\documentclass{article}
\usepackage{ijcai25}

\usepackage{times}
\usepackage{soul}
\usepackage{url}
\usepackage[hidelinks]{hyperref}
\usepackage[utf8]{inputenc}
\usepackage[small]{caption}
\usepackage{graphicx}
\usepackage{amsmath}
\usepackage{amsthm}
\usepackage{booktabs}
\usepackage{algorithm}
\usepackage{algorithmic}
\usepackage[switch]{lineno}

\usepackage{subcaption}
\usepackage{siunitx}
\usepackage{tikz}
\usetikzlibrary{calc}

\usepackage{amsmath,amsfonts,bm}

\def\eqref#1{equation~\ref{#1}}

\def\1{\bm{1}}

\def\vzero{{\bm{0}}}

\def\vtheta{{\bm{\theta}}}
\def\va{{\bm{a}}}
\def\vb{{\bm{b}}}

\def\vv{{\bm{v}}}
\def\vw{{\bm{w}}}
\def\vx{{\bm{x}}}

\def\valpha{{\bm{\alpha}}}
\def\vlambda{{\bm{\lambda}}}

\def\eva{{a}}
\def\evb{{b}}

\def\evw{{w}}

\def\mA{{\bm{A}}}

\def\mD{{\bm{D}}}

\def\mI{{\bm{I}}}

\def\mL{{\bm{L}}}

\def\mU{{\bm{U}}}
\def\mV{{\bm{V}}}
\def\mW{{\bm{W}}}
\def\mX{{\bm{X}}}
\def\mY{{\bm{Y}}}

\def\mLambda{{\bm{\Lambda}}}

\def\mLambda{{\bm{\Lambda}}}

\DeclareMathAlphabet{\mathsfit}{\encodingdefault}{\sfdefault}{m}{sl}
\SetMathAlphabet{\mathsfit}{bold}{\encodingdefault}{\sfdefault}{bx}{n}
\newcommand{\tens}[1]{\bm{\mathsfit{#1}}}

\def\tW{{\tens{W}}}

\def\tTheta{{\bm{\mathsf{\Theta}}}}

\def\gG{{\mathcal{G}}}

\def\gX{{\mathcal{X}}}

\def\emA{{A}}

\def\emU{{U}}

\def\emW{{W}}
\def\emX{{X}}
\def\emY{{Y}}

\newcommand{\etens}[1]{\mathsfit{#1}}

\def\etW{{\etens{W}}}

\newcommand{\R}{\mathbb{R}}

\usepackage{multirow}
\usepackage{mathtools}

\newcommand{\tsym}{\text{sym}}

\newtheorem{example}{Example}
\newtheorem{theorem}{Theorem}
\newtheorem{definition}{Definition}
\newtheorem{proposition}{Proposition}

\newtheorem{lemma}[theorem]{Lemma}

\theoremstyle{remark}

\newtheorem{proposition*}[theorem]{Proposition}

\usepackage{hyperref}
\usepackage[capitalize]{cleveref}
\usepackage{import} 
\usepackage{tikz}
\usetikzlibrary{shapes}

\title{What Can We Learn From MIMO Graph Convolutions?}

\author{
Andreas Roth$^1$
\and
Thomas Liebig$^{1,2}$\\
\affiliations
$^1$ Artificial Intelligence, TU Dortmund University\\
$^2$ Lamarr Insitute for ML and AI, TU Dortmund University\\
\emails
\{andreas.roth, thomas.liebig\}@tu-dortmund.de
}

\begin{document}

\maketitle

\begin{abstract}
Most graph neural networks (GNNs) utilize approximations of the general graph convolution derived in the graph Fourier domain. While GNNs are typically applied in the multi-input multi-output (MIMO) case, the approximations are performed in the single-input single-output (SISO) case. In this work, we first derive the MIMO graph convolution through the convolution theorem and approximate it directly in the MIMO case. We find the key MIMO-specific property of the graph convolution to be operating on multiple computational graphs, or equivalently, applying distinct feature transformations for each pair of nodes. As a localized approximation, we introduce localized MIMO graph convolutions (LMGCs), which generalize many linear message-passing neural networks. For almost every choice of edge weights, we prove that LMGCs with a single computational graph are injective on multisets, and the resulting representations are linearly independent when more than one computational graph is used. Our experimental results confirm that an LMGC can combine the benefits of various methods.
\end{abstract}

\section{Introduction}

Graph neural networks have emerged as an effective method for many challenging applications involving graph-structured data, e.g., molecular prediction~\cite{hu2021ogblsc}. These utilize convolutional operations typically derived from the general graph convolution obtained in the Fourier domain, as given by the convolution theorem~\cite{hammond2011wavelets,bruna2014spectral}. Initially, approximations of the general graph convolution were based on polynomials, e.g., Chebyshev polynomials~\cite{hammond2011wavelets}. The graph convolutional network (GCN)~\cite{kipf2017semisupervised} approximates these polynomials as a first-order localization. Many other message-passing approaches are derived from the GCN~\cite{veličković2018graph,bo2021beyond}. However, these approximations are based on the single-output single-input (SISO) case, where the input and output contain a single feature for each node. GNNs are typically applied in the multi-input multi-output (MIMO) case, where each node has multiple feature channels assigned, and the output also contains multiple features. Extending from the SISO to the MIMO case is achieved by applying these methods for each input and output channel combination and learning distinct parameters~\cite{bruna2014spectral,defferrard2016convolutional,kipf2017semisupervised}. 

Instead of first approximating the graph convolution in the SISO and then extending to the MIMO, we propose directly performing the approximation in the MIMO case to benefit from MIMO-specific properties.
We first derive the general graph convolution in the MIMO case through the convolution theorem and the graph Fourier transform. We find the key property that allows the MIMO-GC to represent arbitrary transformations to be operating on multiple computational graphs or, equivalently, applying distinct linear feature transformations between each pair of nodes.
This form allows a direct approximation in the MIMO case by localizing the aggregation step. The resulting localized MIMO-GC (LMGC) presents a general framework for linear message-passing neural networks (MPNN) that inherits the beneficial properties for multi-channel learning.
While we show that the LMGC can represent most MPNNs, the LMGC cannot represent the graph isomorphism network (GIN)~\cite{xu2018how} due to its non-linear feature transformation. However, we show that LMGCs are injective on multisets for almost every choice of edge weights even for a single computational graph. When further utilizing multiple computational graphs as motivated by the MIMO-GC, we prove that representations are linearly independent for almost every choice of edge weights.
Our experiments confirm the effectiveness of LMGCs for graph-level and node-level tasks.
We summarize our main contributions as follows:

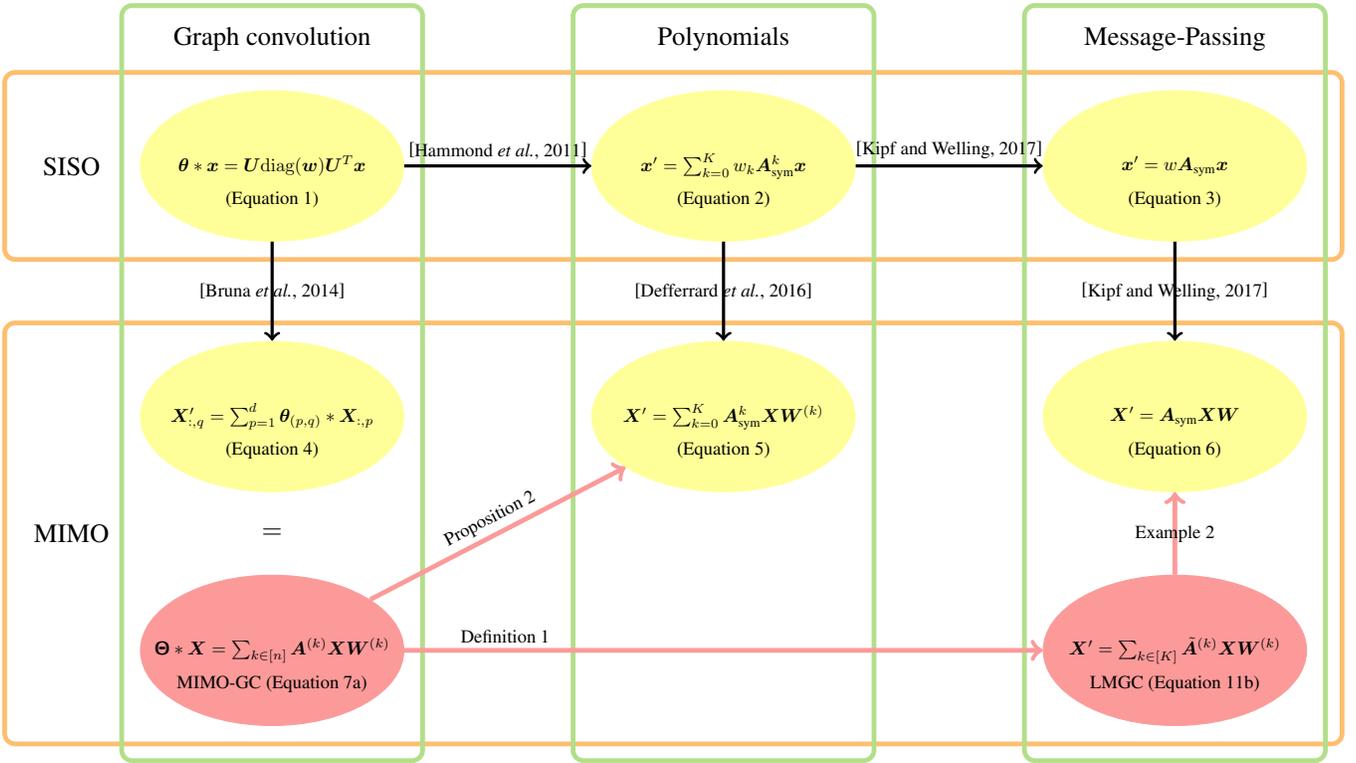
\begin{figure*}[t]
\newcommand{\rectangle}[1]{%
  \begin{tikzpicture}[baseline=-0.5ex]
    \node[minimum size=0.6em, draw=#1, fill=#1, inner sep=0pt] {};
  \end{tikzpicture}%
}
\definecolor{vibrantpink}{HTML}{fb9a99} 
\definecolor{vibrantsand}{HTML}{ffff99} 

\centering
\begin{tikzpicture}[x=\textwidth/4, y=\textwidth/4]
\definecolor{vibrantblue}{HTML}{a6cee3} 
\definecolor{vibrantgreen}{HTML}{b2df8a} 
\definecolor{vibrantpink}{HTML}{fb9a99} 
\definecolor{vibrantorange}{HTML}{fdbf6f} 
\definecolor{vibrantpurple}{HTML}{cab2d6} 
\definecolor{vibrantsand}{HTML}{ffff99} 
\usetikzlibrary{shapes}
\pgfdeclarelayer{background}
\pgfsetlayers{background,main}
\tikzset{
dot/.style = {circle, minimum size=#1,
              inner sep=0pt, outer sep=0pt, draw=black},
dot/.default =13pt  
}
    \tikzstyle{box} = [draw, draw=vibrantgreen, thick,rounded corners,line width=1.8pt]

    \def\ysiso{0.7}
    \def\ymimo{-0.05}
    \def\yours{-0.75}
    \def\xgc{-1.2}
    \def\xmp{1.50}
    \def\xpoly{0.15}
    \def\ymargin{0.28}
    
    \draw[box,draw=vibrantorange] (\xgc-0.8, \ysiso-\ymargin) rectangle (\xmp+0.5, \ysiso+\ymargin); 
    \node at (-1.80, \ysiso) {SISO}; 
        
    \draw[box,draw=vibrantorange] (\xgc-0.8, \yours-\ymargin) rectangle (\xmp+0.5, \ymimo + \ymargin);
    \node at (-1.80, \ymimo-0.35) {MIMO};
    \node[scale=1.0] at (\xgc,\ymimo -0.35) {$=$}; 

    \draw[box] (\xgc -0.45, \yours-\ymargin-0.05) rectangle (\xgc + 0.45, \ysiso+\ymargin+0.2);
    \node at (\xgc, \ysiso+\ymargin+0.1) {Graph convolution};
    \draw[box] (\xmp - 0.45, \yours-\ymargin-0.05) rectangle (\xmp + 0.45,\ysiso+\ymargin+0.2); 
    \node at (\xmp, \ysiso+\ymargin+0.1) {Message-Passing};
    \draw[box] (\xpoly -0.45, \yours-\ymargin-0.05) rectangle (\xpoly + 0.45,\ysiso+\ymargin+0.2);
    \node at (\xpoly, \ysiso+\ymargin+0.1) {Polynomials};
    
    \node[fill=vibrantsand,draw=vibrantsand,minimum width=3.5cm,minimum height=2cm,draw,ellipse] (scgc) at (\xgc,\ysiso) {};
    \node[scale=0.7] at (\xgc, \ysiso - 0.1) {(\autoref{eq:siso_gc})}; 
    \node[scale=0.7] at (\xgc, \ysiso) {$\vtheta*\vx=\mU \mathrm{diag}(\vw)\mU^T\vx$};

    \node[fill=vibrantsand,draw=vibrantsand,minimum width=3.5cm,minimum height=2cm,draw,ellipse] (mcgc) at (\xgc,\ymimo) {};
    \node[scale=0.7] at (\xgc, \ymimo) {$\mX^\prime_{:,q} = \sum_{p=1}^d\vtheta_{(p,q)}*\mX_{:,p}$};
    \node[scale=0.7] at (\xgc, \ymimo - 0.1) {(\autoref{eq:siso_gc_mimo})}; 

    \node[fill=vibrantpink,draw=vibrantpink,minimum width=3.5cm,minimum height=2cm,draw,ellipse] (mimogc) at (\xgc,\yours) {};
    \node[scale=0.7] at (\xgc, \yours - 0.1) {MIMO-GC (\autoref{eq:mcgc})}; 
    \node[scale=0.7] at (\xgc, \yours) {$\tTheta*\mX = \sum_{k\in[n]} \mA^{(k)}\mX\mW^{(k)}$};

    \node[fill=vibrantsand,draw=vibrantsand,minimum width=3.5cm,minimum height=2cm,draw,ellipse] (scmpnn) at (\xmp,\ysiso) {};    
    \node[scale=0.7] at (\xmp, \ysiso) {$\vx^\prime = w \mA_\tsym\vx$}; 
    \node[scale=0.7] at (\xmp, \ysiso - 0.1) {(\autoref{eq:siso_mp})}; 
    \coordinate (x) at (\xmp,\ymimo);
        \node[fill=vibrantsand,draw=vibrantsand,minimum width=3.5cm,minimum height=2cm,draw,ellipse] (a) at (\xmp,\ymimo) {}; 
    \node[scale=0.7] at ($ (x) + (0:0.0) $) {$\mX^\prime = \mA_\tsym\mX\mW$};
        \node[scale=0.7] at (\xmp, \ymimo - 0.1) {(\autoref{eq:gcn})}; 
    \node[fill=vibrantpink,draw=vibrantpink,minimum width=3.5cm,draw,ellipse,minimum height=2cm] (lmcgc) at (\xmp,\yours) {};
    \node[scale=0.7] at (\xmp,\yours) {$\mX^\prime = \sum_{k\in [K]} \tilde{\mA}^{(k)}\mX\mW^{(k)}$}; 
    \node[scale=0.7] at (\xmp,\yours - 0.1) {LMGC (\autoref{eq:lmgc_matrix})}; 


    \node[fill=vibrantsand,draw=vibrantsand,minimum width=3.5cm,minimum height=2cm,draw,ellipse] (scpoly) at (\xpoly,\ysiso) {};
    \node[scale=0.7]  at (\xpoly, \ysiso) {$\vx^\prime = \sum_{k=0}^K w_k \mA_\tsym^k\vx$}; 
    \node[scale=0.7] at (\xpoly, \ysiso - 0.1) {(\autoref{eq:siso_poly})}; 

    \node[fill=vibrantsand,draw=vibrantsand,minimum width=3.5cm,minimum height=2cm,draw,ellipse] (mcpoly) at (\xpoly,\ymimo) {};
    \node[scale=0.7] at (\xpoly, \ymimo) {$\mX^\prime = \sum_{k=0}^K \mA_\tsym^k\mX\mW^{(k)}$};   
    \node[scale=0.7] at (\xpoly, \ymimo - 0.1) {(\autoref{eq:mimo_poly})};

    \draw[->,line width=1.2pt] (scpoly) to node[sloped,above, color=black,scale=0.7] {\cite{kipf2017semisupervised}} (scmpnn);
    \draw[->,line width=1.2pt] (scmpnn) to node[color=black,scale=0.7] {\cite{kipf2017semisupervised}} (a);

    \draw[black,->,line width=1.2pt] (scgc) to node[color=black,scale=0.7] {\cite{bruna2014spectral}} (mcgc);

    \draw[->,line width=1.2pt] (scgc) to node[sloped,above, color=black,scale=0.7] {\cite{hammond2011wavelets}} (scpoly);
    
    \draw[->,line width=1.2pt] (scpoly) to node[color=black,scale=0.7] {\cite{defferrard2016convolutional}} (mcpoly);
    
    \draw[vibrantpink,->,line width=1.8pt] (mimogc) to node[xshift=-2.9cm,above, color=black,scale=0.7] {\autoref{def:lmcgc}}  (lmcgc);
    \draw[vibrantpink,->,line width=1.8pt] (mimogc)  to node[sloped,above, color=black,scale=0.7] {\autoref{prop:polynomials}} (mcpoly);
    \draw[vibrantpink,->,line width=1.8pt] (lmcgc)  to node[color=black,scale=0.7] {\autoref{ex:lmgc_gcn}} (a);

\end{tikzpicture}
\caption{Connections between the graph convolution, polynomial filters, and message-passing approaches in the SISO and the MIMO case. Parts in yellow (\protect\rectangle{vibrantsand}) indicate existing contributions, parts in pink (\protect\rectangle{vibrantpink}) our contributions.}
\label{fig:overview}
\end{figure*}

\begin{itemize}
    \item Based on the convolution theorem, we derive the MIMO graph convolution (MIMO-GC) for node representations with multiple feature channels. A key property of MIMO-GCs is to operate on multiple computational graphs, or equivalently, to apply distinct linear feature transformations for each pair of nodes (Section~\ref{sec:mcgc}).
    \item We introduce the framework of localized MIMO-GCs (LMGCs) by localizing the aggregation step of the MIMO-GC. It merges the key idea of operating on multiple computational graphs with the efficient message-passing scheme (Section~\ref{sec:mc-mpnns}). 
    \item We prove that LMGCs are injective on multisets for a single computational graph and produce linearly independent representations when more than one computational graph is used for almost every choice of edge weights (Section~\ref{sec:mc-mpnns}).
\end{itemize}

\section{Preliminaries}
Let $\gG=(\mathcal{V},\mathcal{E})$ be a connected and undirected graph consisting of a set of $n$ nodes $\mathcal{V}$ and a set of edges $\mathcal{E}$. Let $\mA\in\{0,1\}^{n\times n}$ be the corresponding adjacency matrix with $\emA_{i,j}=1$ if $(i,j)\in\mathcal{E}$ and $0$ otherwise. The diagonal degree matrix is $\mD\in\mathbb{N}^{n\times n}$. The symmetrically normalized adjacency matrix is given by $\mA_\tsym = \mD^{-1/2}\mA\mD^{-1/2}$ and the graph Laplacian by $\mL_\tsym = \mI_n - \mA_\tsym$. Its eigendecomposition is $\mL = \mU\mLambda\mU^T$ where $\mLambda\in\R^{n\times n}$ is a diagonal matrix containing its eigenvalues, and $\mU\in\R^{n\times n}$ is an orthonormal matrix containing the corresponding eigenvectors as columns. 
We refer to a vector $\vx\in\R^n$ as a single-channel graph signal and to a matrix $\mX\in\R^{n\times d}$ as a multi-channel graph signal. These can be initial features or expressive and informative node embeddings.
In the graph domain, the Fourier base is given by the eigenvectors $\mU^T$ of the graph Laplacian. Thus, the Fourier transformation $F = \mU^T$ is performed by projecting a graph signal onto the eigenvectors, and its inverse transformation is given by $F^{-1} = \mU$.
 We further refer to $\mU_{i,:}\vx\in\R$ as the component of $\vx$ corresponding to vector $\mU_{i,:}$.

\subsection{Graph Convolutions}
\label{sec:single-channel}
Given a graph signal, the graph convolution or a similar method derived from it are designed to obtain a more informative graph signal. In the SISO case, the input and output are single-channel graph signals, while in the MIMO case, they are multi-channel graph signals. Graph neural networks (GNNs) are typically constructed by interleaving these operations with non-linear activation functions. The following derivations and approximations of the graph convolution are visualized in \cref{fig:overview}.

The general graph convolution is defined in the SISO case through the convolution theorem using the graph Fourier transform as
\begin{equation}
\label{eq:siso_gc}
\begin{split}
    \vtheta*\vx
    &= \mU\textrm{diag}(\vw)\mU^T\vx
\end{split}
\end{equation}
where $\vw = \mU^T\vtheta\in\R^n$~\cite{hammond2011wavelets,bruna2014spectral}. As $\vx$ and $\vtheta * \vx$ are single-channel signals, we will refer to this as the SISO graph convolution (SISO-GC). 

Due to the runtime and memory complexity and inability to apply the same graph Fourier transform across graphs, most GNNs utilize approximations of the SISO-GC. Polynomials in $\mA_\tsym$ (or equivalently $\mL_\tsym$) provide a $K$-localized approximation
\begin{equation}
\label{eq:siso_poly}
    \vtheta * \vx \approx \sum_{k=0}^K w_{(k)}\mA_\tsym^k\vx 
\end{equation}
of the SISO-GC where $w_k\in\R$ are scalars for $k\in[K]$~\cite{hammond2011wavelets}. Examples of such approximations are Chebyshev~\cite{hammond2011wavelets} and Cayley polynomials~\cite{levis2019cayleynets}. 

Similarly, the graph convolutional network (GCN)~\cite{kipf2017semisupervised} was derived as a first-order localization
\begin{equation}
\label{eq:siso_mp}
    \vtheta * \vx \approx w\mA_\tsym\vx
\end{equation}
of SISO polynomials using a single parameter $w\in\R$. 

The graph convolution has not yet been derived for the MIMO case. Instead, following Bruna et al.~\shortcite{bruna2014spectral}, the graph convolution and the described approximations are extended to the MIMO case by applying it to each combination of input channel $p\in[d]$ and output channel $q\in[c]$. For the graph convolution, the output
\begin{equation}
\label{eq:siso_gc_mimo}
    \mX_{:,q}^\prime = \sum_{p=1}^d\vtheta_{(p,q)} * \mX_{:,p}
\end{equation}
is obtained by defining distinct filters $\theta_{(p,q)}\in\R^n$.

SISO polynomials are equivalently extended to the MIMO case by applying distinct parameters $\emW^{(k)}_{p,q}\in\R$ for each combination of input and output channels~\cite{defferrard2016convolutional}. Based on Equation~\ref{eq:siso_poly}, we have
\begin{equation}
\label{eq:mimo_poly}
    \mX^\prime_{:,q} = \sum_{p=1}^d\sum_{k=0}^K \emW^{(k)}_{p,q}\mA_\tsym^k\mX_{:,p}
\end{equation}
where $\mW^{(k)}\in\R^{d\times c}$.

Equivalently, the GCN is applied in the MIMO case using distinct parameters $\emW_{p,q}\in\R$ for each combination of input channel $p$ and output channel $q$~\cite{kipf2017semisupervised}. This led to the typical form of
\begin{equation}
\label{eq:gcn}
\mX^\prime_{:,q} = \sum_{p=1}^d \emW_{p,q}\mA_\tsym\mX_{:,p} = [\mA_\tsym\mX\mW]_{:,q}\, .
\end{equation}
Most other message-passing methods were then further derived from the GCN. 
In this work, we show the advantages of directly obtaining the graph convolution and approximations in the MIMO case.
\section{MIMO Graph Convolution}
\label{sec:mcgc}
We now consider the MIMO case. Let $\mX\in\mathbb{R}^{n\times d}$ be a multi-channel graph signal with $d$ channels for each node. The multi-channel output signal $\mY\in\mathbb{R}^{n\times c}$ can have a different number of channels $c$. We first derive the general graph convolution for the MIMO through the convolution theorem~\cite{neil1963convolution} and the graph Fourier transform. The filter $\tTheta\in\mathbb{R}^{n\times c \times d}$ contains the necessary element-wise mappings from $d$ to $c$ dimensions. To the best of our knowledge, this has not yet been derived.


\begin{theorem}[MIMO Graph Convolution (MIMO-GC)]
\label{thrm:mcgc}
Let $\mX\in\mathbb{R}^{n\times d}$, $\tTheta\in\mathbb{R}^{n\times c \times d}$, and the Fourier transform $F = \mU^T\in\mathbb{R}^{n\times n}$ be given by the eigenvectors of the graph Laplacian $\mLambda$. Then, their convolution is given as
\begin{subequations}
    \begin{align}
    (\tTheta * \mX)(i) &= \left[\sum_{k=1}^{n} \mA^{(k)}\mX\mW^{(k)}\right]_{i,:} \label{eq:mcgc} \\
    &= \sum_{j=1}^{n} \mW_{(i,j)} \mX_{j,:} \in \mathbb{R}^{c} \label{eq:mcgc_node}
    \end{align}
    \end{subequations}    
    where $\mA^{(k)} = \mU_{:,k} (\mU_{:,k})^T\in\R^{n\times n}$, $\mW^{(k)} = \left(F(\tTheta)_{k,:,:}\right)^T\in\mathbb{R}^{d\times c}$ and $\mW_{(i,j)} = (\sum_{k=1}^n \emU_{i,k}\emU_{j,k}\mW^{(k)})^T\in\mathbb{R}^{c\times d}$.
\end{theorem}
We provide all detailed proofs as supplementary material. The MIMO-GC is unique because it does not require additional definitions from us. As such, MIMO operations on graphs should closely approximate the MIMO-GC. 
We note that the MIMO-GC is equivalent to extending the SISO-GC to multi-channel signals by applying it to every pair of input and output channels, as introduced by Bruna et al.~\shortcite{bruna2014spectral}. The MIMO-GC can be interpreted in two ways.

Based on \cref{eq:mcgc}, each $\mA^{(k)}\in\R^{n\times n}$ can be seen as a fully connected computational graph with edge weights $\emA^{(k)}_{i,j} = \emU_{i,k}\cdot\emU_{j,k}\in\R$ given by the corresponding Fourier basis vector $\mU_{:,k}$. This form is also similar to multi-head self-attention~\cite{vaswani2017attention}. However, they normalize edge weights by the softmax activation, preventing them from being orthogonal across heads.
The corresponding parameter matrix $\mW^{(k)}$ specifies how much this component is amplified or damped from each input channel to each output channel. Utilizing $n$ computational graphs allows the MIMO-GC to amplify distinct components for each output channel. Assuming all components are present in the input signal, the MIMO-GC can produce any output signal:

\begin{proposition}[Universality of the MIMO-GC]
\label{prop:universality}
    For any $\mX\in\R^{n\times d}$ with $\mU^T\mX \neq_{em} 0$ element-wise non-zero and any $\mY\in\R^{n\times c}$, there exists a $\tTheta\in\R^{n\times c\times d}$, such that
    \begin{equation}
    \tTheta * \mX = \mY\, .
    \end{equation}
\end{proposition}

Based on \cref{eq:mcgc_node}, the MIMO-GC can also be interpreted as applying distinct feature transformation $\mW_{(i,j)}$ for each pair of nodes. Each $\mW_{(i,j)}$ is a unique linear combination of a shared set of $n$ feature transformations. Relatedly, utilizing distinct feature transformations was recently popularized as Neural Sheaf Diffusion~\cite{hansen2020sehaf,bodnar2022neural}. The MIMO-GC provides an additional theoretical justification for such methods.

However, computing the MIMO-GC exactly is typically not desirable, as with the SISO-GC. It is inherently transductive, as the graph Fourier transform is graph-dependent, and thus, a learned filter cannot be applied to novel or changed graphs. 
Most importantly, the computational complexity of the graph convolution scales quadratically with the number of nodes:

\paragraph{Computational Complexity}
Equivalently to computing the SISO-GC exactly, the total complexity of the MIMO-GC is dominated by the graph Fourier transform as it requires dense matrix multiplications. The overall complexity is thus $\mathcal{O}(n^2\cdot c\cdot d)$. 

\paragraph{Benefiting from the MIMO-GC}
Instead of directly computing the MIMO-GC, we aim to improve the approximations previously derived from the SISO-GC, which were then extended to the MIMO case. 
We first confirm that these MIMO polynomials are also approximations of the MIMO-GC with constraints on the allowed filters $\tTheta$:
\begin{proposition} [Every MIMO polynomial filter is a MIMO-GC with a specific filter]
\label{prop:polynomials}
    Let $\mX\in\mathbb{R}^{n\times d}$ for some $d\in\mathbb{N}$. For any $\mV^{(0)},\dots,\mV^{(K)}\in\mathbb{R}^{d\times c}$ with $c,K\in\mathbb{N}$, there exists a $\tTheta_{\textrm{poly}}\in\mathbb{R}^{n\times c \times d}$, such that 
    \begin{equation}
    \sum_{k=0}^K \mA_\tsym^k\mX\mV^{(k)} = \tTheta_{\textrm{poly}} * \mX\, .
    \end{equation}
\end{proposition}

As one such example of a first-degree polynomial, the GCN is a MIMO-GC with specific constraints on $\tTheta$:

\begin{example}[GCN is a MIMO-GC]
\label{ex:gcn_mimo_gc}
    Let $\mX\in\mathbb{R}^{n\times d},\mV\in\mathbb{R}^{d\times c}$. Then, 
    \begin{equation}
    \begin{split}
        \mA_\tsym\mX\mV 
        &= \sum_{k=1}^n \lambda_j\mU_{:,k}(\mU_{:,k})^T\mX\mV \\
        &= \sum_{k=1}^n \mU_{:,k}(\mU_{:,k})^T\mX\mW^{(k)} \\
        &= \tTheta_{\textrm{GCN}} * \mX
        \end{split}
    \end{equation}
    where $\mW^{(k)} = \lambda_j\mV$ and corresponding $\tTheta_{\textrm{GCN}}\in\R^{n\times c\times d}$.
\end{example}

As a first step, the MIMO-GC helps us with the understanding of properties of various approximations and can consequently improve these approximations. Based on \cref{ex:gcn_mimo_gc}, the GCN utilizes a single shared parameter matrix $\mV$ across all components. Each component is then amplified according to its respective eigenvalue, which is shared across all combinations of input and output channels. 
Other message-passing operations may utilize a different matrix $\tilde{\mA}$ instead of $\mA_\tsym$. However, as using any single computational graph $\tilde{\mA}$ can be similarly decomposed, the amplification of components is fixed and shared across all feature channels for any given $\tilde{\mA}$.
We refer to this phenomenon as shared component amplification (SCA)
When repeatedly applying such filters or message-passing operations, SCA leads to the well-known phenomenon of over-smoothing and, more generally, rank collapse~\cite{Oono2020Graph,roth2023rank}. We provide further details on this phenomenon in our appendix.

Contrarily, the MIMO-GC requires multiple computational graphs to amplify different components across feature channels. Equivalently, applying distinct feature transformations for each node pair can improve approximations. Developing approximations with these properties can lead to more effective learning on graph-structured data.


\section{Localized MIMO Graph Convolutions}
\label{sec:mc-mpnns}
Based on \cref{eq:mcgc_node}, we localize the MIMO-GC by aggregating over the neighboring nodes instead of all nodes of a given graph:

\begin{definition}
\label{def:lmcgc}
We define the Localized MIMO Graph Convolution (LMGC) as: 
\begin{subequations}
\begin{align}
\label{eq:lmcgc}
    \vx_{(i)}^\prime &= \sum_{v_j\in N_i} \mW_{(i,j)}\vx_{(j)} \\
    &= \left[\sum_{k\in[K]}\tilde{\mA}^{(k)}\mX\mW^{(k)}\right]_{i,:} \label{eq:lmgc_matrix}
    \end{align}
\end{subequations}
where $K\in\mathbb{N}$ and each $\mW_{(i,j)} = \sum_{k\in[K]}\alpha_{(k)}^{(i,j)}\mW^{(k)}\in\R^{c\times d}$ is linear combination based on $\alpha_{(1)}^{(i,j)},\dots\alpha_{(K)}^{(i,j)}\in\R$ and $\mW^{(1)},\dots,\mW^{(K)}\in\R^{d\times c}$. The entries $\emA_{i,j}^{(k)}=\alpha_{(k)}^{(i,j)}$ are given by the corresponding coefficients.
\end{definition}
In this definition, the number of terms $K$ and the coefficients or edge weights $\alpha_{(k)}^{(i,j)}$ can be freely chosen, which allows methods that do not use the expensive eigenvector computation. The LMGC is permutation equivariant if the coefficients $\alpha_{(k)}^{(i,j)}$ are also equivariant, for example, when derived from a function of the nodes $v_i$ and $v_j$. The LMGC can also be applied across different graphs and for directed graphs.
As with MIMO-GCs, the LMGC can be equivalently restated as operating on $K$ computational graphs. The edge weights of the $k$-th computational graph are given by $\alpha_{(k)}^{(i,j)}$.
Consequently, the LMGC can represent many linear MPNNs for different values for $\alpha_{(k)}^{(i,j)}$. We provide three examples below:

\begin{example}[GCN~\cite{kipf2017semisupervised}]
\label{ex:lmgc_gcn}
Let $\mV\in\mathbb{R}^{c\times d}$ be a feature transformation. The update step
\begin{equation}
    \vx_{(i)}^\prime = \sum_{v_j\in N_i} \frac{1}{\sqrt{d_i}\sqrt{d_j}}\mV \vx_{(j)} 
\end{equation}
is an LMGC with $K=1$, $\mW^{(1)} = \mV$, and $\alpha_{(1)}^{(i,j)} = \frac{1}{\sqrt{d_i}\sqrt{d_j}}$ where $d_i,d_j\in\mathbb{N}$ are the degrees of nodes $v_i$ and $v_j$, respectively.    
\end{example}
As the MIMO-GC is similar to multi-head self-attention, the LMGC is related to local multi-head attention-based methods while allowing for more flexible attention scores, i.e., scores do not need to sum to one for every node:
\begin{example}[GAT~\cite{veličković2018graph}]\label{ex:gat} Let $H$ be the number of heads, $\mV^{(h)}$ the linear transformation of head $h\in[H]$, and $a^{(i,j)}_{(h)}\in\mathbb{R}$ the attention score between nodes $v_i$ and $v_j$. The update step
    \begin{equation}
    \vx_{(i)}^\prime = \sum_{h\in[H]} \sum_{v_j\in N_i} a^{(i,j)}_{(h)} \mV^{(h)}\vx_{(j)}
    \end{equation}
    is an LMGC with $K=H$, $\mW^{(h)} = \mV^{(h)}$ and $\alpha_{(h)}^{(i,j)} = a^{(i,j)}_{(h)}$. 
\end{example}
The LMGC can also represent gating mechanisms, e.g., the GatedGCN~\cite{dwivedi2023benchmarking} or neural sheaf diffusion~\cite{hansen2020sehaf}.

The general form of the LMGC allows for a more focused development of novel and powerful methods. With specific choices of $\alpha_{(k)}^{(i,j)}$, the LMGC can model a symmetric or directed flow of information and can construct anisotropic or isotropic messages. 

\paragraph{Theoretical Properties}
Studying theoretical properties of LMGCs reduces to studying the effects of coefficients $\alpha_{(k)}^{(i,j)}$. For example, the LMGC cannot represent non-linear feature transformations, which are typically used to ensure injectivity, e.g., by GIN~\cite{xu2018how}. This allows GNNs to match the expressivity of the Weisfeiler-Leman graph isomorphism test~\cite{leman1968reduction}, a key property for graph-level tasks. However, we find that any LMGC with $K > 0$ computational graphs is also injective for almost every choice coefficients $\alpha_{(k)}^{(i,j)}$ without requiring a non-linear feature transformation:





\begin{proposition}[Injectivity]
\label{prop:injectivity}
Let $f(\vx_{(i)},N_i) = \sum_{\vx_{(j)}\in N_i}\mW_{(i,j)}\vx_{(j)}$ be an LMGC with $K\geq1$ and $\gX$ a countable set. Then, $f(\vx_{(p)}, \gX_p)$ is injective for finite multisets $\gX_p\subset \gX$ and elements $\vx_{(p)}\in\gX$ for a.e. choice of coefficients $\alpha_{(k)}^{(i,j)}$ and a.e. $\mW^{(k)}$ for all $k\in[K]$. 
\end{proposition}

Different components can be amplified across feature channels when further using $K>1$ computational graphs. 
The resulting node representations are linearly independent for almost every choice of coefficients $\alpha_{(k)}^{(i,j)}$.
This prevents the shared component amplification of methods utilizing a single computational graph. 

\begin{proposition}[Linear Independence]
\label{prop:lin_indepence}
Let $f(\vx_{(i)},N_i) = \sum_{\vx_{(j)}\in N_i}\mW_{(i,j)}\vx_{(j)}$ be an LMGC with $K>1$ and $\gX$ a countable set. Then, $f(\vx_{(i)}, \gX_1)$ is linearly independent to $f(\vx_{(j)}, \gX_2)$ for all finite multisets $\gX_1,\gX_2\subset \gX$ with $\gX_1 \neq c\cdot\gX_2$ for any $c\in\mathbb{N}$ and elements $\vx_{(i)},\vx_{(j)}\in\gX$ for a.e. choice of coefficients $\alpha_{(k)}^{(i,j)}$ and a.e. $\mW^{(k)}$ for all $k\in[K]$. 
\end{proposition}

This result aligns with previous findings that identified cases where multiple computational graphs can ensure linearly independent representations~\cite{roth2024preventing}.
Importantly, each $\alpha_{(k)}^{(i,j)}$ can be independently obtained, e.g., by a function $\alpha_{(k)}^{(i,j)} = \phi_k(\vx_{(i)},\vx_{(j)})\in\mathbb{R}$ of the corresponding node states. Many functions $\phi_k$ satisfy \cref{prop:injectivity} and \cref{prop:lin_indepence}. A neural network can then approximate such a function. 
As a negative example of such a functions, softmax-activated attention scores do not satisfy the a.e. condition as the space of scores forms a measure-zero set, e.g., for GAT~\cite{veličković2018graph} and the more powerful GATv2~\cite{brody2022how}. As has been pointed out by several works~\cite{xu2018how}, such methods cannot distinguish multisets of different multiplicities, e.g., when $\gX_1 = \{\{\vx_1\}\}$ and $\gX_2 = \{\{\vx_1,\vx_1\}\}$. Other methods, such as FAGCN~\cite{bo2021beyond} and GGCN~\cite{yan2022two}, proposed to apply the tanh activation function instead, which does not constrain the outputs to a measure-zero set.

Thus, an LMGC can incorporate the advantages of attention-based by filtering incoming messages and preventing the shared component amplification across feature channels by utilizing multiple computational graphs. At the same time, it applies linear feature transformations and can be injective on multisets, as in GIN.

\paragraph{An LMGC Instantiation}
When constructing an LMGC instantiation, only the number of computational graphs $K$ and the coefficients $\alpha_{(k)}^{(i,j)}$ for all $k\in[K]$ need to be defined.
For our empirical study, we define a simple LMGC instantiation as a mix of GATv2 and FAGCN. We define the coefficients as 
\begin{multline}
\label{eq:lmgc_exp}
    \alpha_{(k)}^{(i,j)} = \phi_k(\vx_{(i)},\vx_{(j)}) \coloneqq \sigma_2(\vv_{(k)}^T\sigma_1(\mW^{(1)}\vx_{(i)} ||\\\dots|| \mW^{(K)}\vx_{(i)} || \mW^{(1)}\vx_{(j)} ||\dots||\mW^{(K)}\vx_{(j)}))
\end{multline}
where $\vv_{(k)}\in\R^{2\cdot K\cdot c}$ are learnable vectors for $k\in[K]$, $\sigma_1$ is the LeakyReLU activation and $\sigma_2$ is the tanh activation function. The execution time is slightly favorable compared to GATv2, as we do not normalize the messages.


\begin{table}
\centering
\begin{tabular}{lr}
\toprule
Method  & MSE \\
\midrule
GATv2 & $0.12\pm0.04$ \\
FAGCN & $0.68\pm0.02$ \\
ACM & $0.49\pm0.02$ \\
GIN & $\underline{0.08}\pm0.03$\\
LMGC & $\mathbf{25\cdot10^{-9}}\pm86\cdot10^{-11}$ \\
\bottomrule
\end{tabular}
\caption{Results for the universality task. Given representations $\mX,\mY$ and a graph $\mA$, one layer of each method is optimized to approximate the function $f(\mX,\mA) = \mY$. Average and standard deviation of the minimal mean-squared error (MSE) during optimization. Best MSE in \textbf{bold}, second-best \underline{underlined}.}
\label{tab:approx}
\end{table}

\begin{table*}[t]
\centering
\begin{tabular}{lrrrrr}
\toprule
Method & Basic & + LapPE & + Jumping Knowledge & + Residual & + All three \\
\midrule
GATv2 & $0.377\pm0.024$ & $0.341\pm0.040$ & $0.388\pm0.017$ & $0.311\pm0.016$ & $0.294\pm0.019$ \\
FAGCN & $0.365\pm0.018$ & $0.349\pm0.038$ & $0.352\pm0.042$ & $0.289\pm0.019$ & $0.232\pm0.012$ \\
ACM & $0.278\pm0.006$ & $0.281\pm0.019$ & $0.288\pm0.008$ & $0.266\pm0.017$ & $0.238\pm0.006$ \\
GIN & $\underline{0.272}\pm0.009$ & $\underline{0.259}\pm0.012$ & $\underline{0.267}\pm0.020$ & $\underline{0.240}\pm0.005$ & $\underline{0.228}\pm0.014$\\
LMGC & $\mathbf{0.241}\pm0.018$ & $\mathbf{0.234}\pm0.009$ & $\mathbf{0.233}\pm0.019$ & $\mathbf{0.215}\pm0.006$ & $\mathbf{0.203}\pm0.004$\\
\bottomrule
\end{tabular}
\caption{Test MAE results on ZINC12k. LapPE indicates that a Laplacian position encoding is concatenated to the initial features. For Jumping Knowledge, the channel-wise maximum value after each iteration is used for each after the message-passing steps. Residual indicates that the input to each message-passing step is added to its output. With + All three, these three techniques are simultaneously applied. Best scores in \textbf{bold}, second-best \underline{underlined}.}
\label{tab:zinc12k}
\end{table*}

\begin{table*}[t]
\centering
\begin{tabular}{lrrrrrr}
\toprule
Method & Texas & Cornell & Wisconsin & Film & Chameleon & Squirrel \\
\midrule
GATv2 & $71.6\pm1.0$ & $66.1\pm0.6$ & $79.1\pm2.0$ & $35.1\pm0.2$ & $\underline{47.1}\pm0.3$ & $\underline{35.1}\pm0.2$ \\
FAGCN & $\underline{73.5}\pm1.8$ &  $\underline{68.1}\pm1.9$ & $\underline{80.2}\pm1.8$ & $\underline{36.0}\pm0.3$ & $46.9\pm0.5$ & $34.6\pm0.3$ \\
ACM & $72.3\pm0.4$ & $65.1\pm0.7$ & $74.2\pm0.9$ & $35.8\pm0.3$ & $45.5\pm0.9$ & $34.5\pm0.1$\\
GIN & $70.5\pm1.1$ & $66.1\pm1.0$ & $79.0\pm0.6$ & $34.1\pm0.3$ & $46.1\pm0.4$ & $34.6\pm0.5$ \\
LMGC & $\mathbf{74.2}\pm2.2$ & $\mathbf{68.9}\pm2.2$ & $\mathbf{81.4}\pm1.1$ & $\mathbf{36.3}\pm0.4$ & $\mathbf{49.8}\pm0.8$ & $\mathbf{35.9}\pm0.5$ \\
\bottomrule
\end{tabular}
\caption{Test accuracy on heterophilic node classification tasks. Best scores in \textbf{bold}, second-best \underline{underlined}. All models contain at most \num{100000} parameters and the same hyperparameter optimization was applied.}
\label{tab:hetero}
\end{table*}

\section{Related Work}
We now describe previous works related to various parts of the MIMO-GC and the LMGC.
\paragraph{Graph Convolutions}
Bruna et al.~\shortcite{bruna2014spectral} extend the SISO-GC to the MIMO case by utilizing a filter between all pairs of input and output channels. This extension is equivalent to the MIMO-GC directly derived through the convolution theorem. Approximations are derived in the SISO case and mapped to the MIMO case using the same procedure afterward.
Hammond et al.~\shortcite{hammond2011wavelets} propose to approximate the SISO-GC using Chebyshev polynomials in the SISO case. 
Defferrard et al.~\shortcite{defferrard2016convolutional} employ separate filters for pairs of input and output channels to extend Chebyshev polynomials to the MIMO case.
Sandryhaila~\shortcite{sandryhaila2013discrete} define general polynomial graph filters for the SISO case. Using the same procedure, Gama et al.~\shortcite{gama2018mimo} extend these polynomial graph filters to the MIMO case. 
Kipf and Welling~\shortcite{kipf2017semisupervised} derive the GCN as a 1-localized approximation of the SISO Chebyshev polynomials. They equivalently extend it to the MIMO case afterward by applying separate parameters for each combination of input and output channels. Most other MPNNs are derived from the GCN to mitigate various shortcomings~\cite{veličković2018graph,xu2018how,roth2022transforming}.
Directly approximating the MIMO-GC allows us to benefit from MIMO-specific properties of the graph convolution.

\paragraph{MIMO Improvements}
While most MPNNs are applied to the MIMO case, many of these are well-known to be unable to amplify distinct components across channels, a phenomenon known as over-smoothing~\cite{Oono2020Graph}, over-correlation~\cite{jin2022feature}, or rank collapse~\cite{roth2023rank,roth2024simplifying}. Various methods have been proposed to improve multi-channel learning within MPNNs. Luan et al.~\shortcite{luan2022revisiting} propose to apply separate graph filters for different feature channels. In ADR-GNNs~\cite{eliasof2023adrgnn}, feature channels are separately aggregated using channel-specific edge weights.
Other works similarly propose to apply distinct filters across channels~\cite{liu2025gpnet}. Zhou et al.~\shortcite{zhou2020multi} propose the multi-channel graph neural network that obtains multiple computational graphs through a pooling operation and learns interaction scores between graphs. Utilizing multiple computational graphs has been extensively studied in mitigating over-smoothing and representational rank collapse~\cite{roth2024preventing}. Applying different linear transformations between pairs of nodes has also been derived within neural sheaf diffusion~\cite{hansen2020sehaf,bodnar2022neural}.
As the MIMO-GC and LMGC naturally allow multi-channel learning, these frameworks can be closer aligned as approximations of the MIMO-GC. The LMGC can equivalently be interpreted as message-passing on multigraphs. Butler et al.~\shortcite{butler2023convolutional} introduced convolutional multigraph neural networks that utilize polynomial filters on multigraphs.

\paragraph{Approaches Related to the LMGC}
The LMGC is closely related to several existing methods. As described in \cref{ex:gat}, multi-head attention-based methods like GAT~\cite{veličković2018graph} and GATv2~\cite{brody2022how} are LMGCs with constraints on the attention scores by applying the softmax activation. By lifting this constraint, LMGCs can be injective on multisets (\cref{prop:injectivity}). Several other methods have been proposed to replace the softmax activation. The FAGCN~\cite{bo2021beyond} instead applies the tanh activation function to amplify high-frequencies or low-frequencies. Similarly, the GGCN~\cite{yan2022two} allows learning of signed edge weights.  
Other studies considered replacing the softmax activation function within transformers and self-attention modules. Wortsman et al.~\shortcite{wortsman2023replacing} apply the ReLU activation in vision transformers. Saratchandran et al.~\shortcite{saratchandran2024rethinking} found empirical success using polynomial activation functions for self-attention. 
However, as self-attention typically considers a fully connected graph, these works did not study distinguishing structural differences. 
Contrarily, \cref{prop:injectivity} shows that differences in the number of neighbors can be distinguished without the softmax activation.

\section{Experiments}
\label{sec:experiments}
We now want to confirm the beneficial properties of LMGCs. As the LMGC can match the expressive power of GIN, we want to evaluate whether it can match the performance of GIN for graph-level tasks. We also evaluate whether the LMGC can match the performance of attention-based methods for node-level tasks. All experiments are run on an H100 GPU. Additional details on all models, datasets, and hyperparameters are provided as supplementary material.\footnote{Our implementation is available at \href{https://github.com/roth-andreas/mimo-graph-convolutions}{https://github.com/roth-andreas/mimo-graph-convolutions}.}
\subsection{Methods}
We consider the following four message-passing methods across all experiments. We conduct all results ourselves using the same hyperparameter ranges across methods. 

\paragraph{GATv2} This method extends GAT with dynamic attention~\cite{brody2022how}. Attention-based methods are particularly effective for node-level prediction tasks due to their ability to filter information. We utilize the standard implementation that corresponds to an LMGC with $\alpha_{(k)}^{(i,j)} = \sigma_2(\vv_{(k)}^T\sigma_1(\mW^{(k)}\vx_{(i)} + \mW^{(k)}\vx_{(j)}))$ where $\sigma_1$ is the LeakyReLU activation function and $\sigma_2$ is the node-wise softmax activation function and $\vv^{(k)}\in\R^c$ is a learnable vector. We set the number of heads to $K=4$ for all experiments.
\paragraph{FAGCN} This method was designed for heterophilic node classification tasks by allowing for negative edge weights~\cite{bo2021beyond}. Stated in the LMGC framework, we evaluate a method that sets $K=1$ and $\alpha_{(1)}^{(i,j)}=\frac{\sigma(\vv[\vx_{(i)}||\vx_{(j)}])}{\sqrt{d_i}\sqrt{d_j}}$ where $\sigma$ is the tanh activation function, $\vv\in\R^{2\cdot d}$ is a learnable vector and $d_i,d_j$ are the degrees of nodes $v_i$ and $v_j$, respectively. Despite the tanh activation, FAGCN is not always injective due to the degree normalization.
\paragraph{ACM} Written in the LMGC framework, the adaptive channel mixing (ACM)~\cite{luan2022revisiting} proposes to utilize $\tilde{\mA}^{(1)} = \mA_\tsym$ for amplifying low-frequency components and $\tilde{\mA}^{(2)} = \mL_\tsym$ for amplifying high-frequency components. They further propose a third computational graph $\tilde{\mA}^{(3)} = \mI$, which we utilize whenever residual connections are used in an experiment. 
\paragraph{GIN} For graph-level tasks, the GIN~\cite{xu2018how} is particularly effective as it can match the expressivity of the WL-test due to the non-linear feature transformation. As the non-linear feature transformation, we apply a two-layer MLP with ReLU activations for all experiments.
\paragraph{LMGC} As the LMGC can combine the favorable properties of the other three methods, we utilize the instantiation of the LMGC as described in \cref{eq:lmgc_exp}. As with GATv2, we set the total number of heads to $K=4$ for all experiments.

\subsection{Universality}
Based on \cref{prop:universality}, the MIMO-GC can represent almost every mapping $\tTheta * \mX = \mY$ with $\mX\in\R^{n\times d}$ and $\mY\in\R^{n\times c}$. We evaluate the ability of message-passing approaches to approximate such a transformation. We sample a random undirected and connected Erdős–Rényi graph~\cite{erdos1959on} with $n=16$ nodes and an edge probability of $p=10\%$. We set $d=c=16$. 
Similarly, $\mX\in\R^{n\times d}$ and $\mY\in\R^{n\times c}$ are randomly sampled with $\emX_{i,j}\sim \mathcal{N}(0,1)$ and $\emY_{i,j}\sim\mathcal{N}(0,1)$. 
We apply a single message-passing layer as $f(\mX,\mA)$. We minimize the mean-squared error (MSE) between $f(\mX,\mA)$ and $\mY$ using the Adam optimizer for $\num{40000}$ steps. The learning rate is tuned in $\{0.03,0.01,0.003\}$. 

The minimum achieved approximation error averaged over three runs is presented in \cref{tab:approx}. LMGC achieves a significantly lower error than GATv2, FAGCN, and GIN. As the MIMO-GC can represent such a function exactly, LMGCs benefit from this property as a close approximation. While these improved capabilities come with the risk of overfitting, we expect LMGCs to be particularly beneficial for challenging tasks.

\subsection{Graph-Level Prediction}
GIN is typically used for graph-level tasks due to its expressive power. Based on \cref{prop:injectivity}, we now want to validate that the LMGC can match these results empirically. 
We consider the challenging ZINC12k dataset~\cite{sterling2015zinc}. It consists of around $\num{12000}$ molecular graphs, with the task being to predict the constrained solubility of each molecule. We integrate all models into the implementation of GraphGPS~\cite{rampasek2022recipe} and the Long Range Graph Benchmark~\cite{dwivedi2022long}. Based on Toenshoff et al.~\shortcite{toenshoff2024where}, we optimize the number of layers in $\{6,8,10\}$ and the learning rate in $\{0.001,0.0003,0.0001\}$ using a grid search. Each model utilizes at most $\num{100000}$ parameters to ensure fairness.

In \cref{tab:zinc12k}, we present results of a detailed study in which we combine these base message-passing methods with various other established techniques. These techniques are Laplacian positional encoding (LapPE)~\cite{kreuzer2021rethinking}, jumping knowledge~\cite{xu2018representation} and residual connections~\cite{he2016deep}. We find all methods to benefit from these techniques, with the LMGC achieving the best results in all cases. We provide additional results, including runtimes and training losses, as supplementary material.

\subsection{Node Classification}
While expressivity is a key property for graph-level tasks, attention-based methods typically outperform GIN on node-level tasks due to their ability to filter messages~\cite{brody2022how}. Thus, we also evaluate whether the LMGC can match the performance of GATv2 and FAGCN on these tasks. We consider six heterophilic benchmark datasets for node classification: Texas, Cornell, Wisconsin, Film, Chameleon, and Squirrel. We use the ten splits into train, validation, and test sets proposed by Pei et al.~\shortcite{pei2020geom}. We integrate all models into the implementation from Rusch et al.~\shortcite{rusch2023gradient}. As with ZINC, each model uses at most \num{100000} parameters. For each method, we tune the learning rate in $\{0.01,0.003,0.001\}$ and dropout ratio in $\{0.0,0.25,0.5\}$ using a grid search, as these affected the results the most. Based on the optimal hyperparameters for the validation set, we rerun each method five times for all ten splits and report average test results. 

These average test accuracies are presented in Table~\ref{tab:hetero}. GIN achieves the lowest accuracy, and LMGC achieves the highest accuracy across all tasks. While the differences are only a few percentage points, these experiments confirm that the LMGC can combine the benefits of GATv2, FAGCN, GIN, and the MIMO-GC into a single model.

\section{Conclusion}
This work derives the MIMO graph convolution (MIMO-GC) using the convolution theorem and emphasizes the advantages of approximating the graph convolution in the MIMO case rather than the SISO case. A key property of the MIMO-GC is operating on multiple computational graphs or equivalently applying distinct linear transformations for each node pair. We have proven that the localized form is injective and results in linearly independent representations for almost every choice of edge weights. 
Due to our direct theoretical derivation from the MIMO-GC and the generality of the LMGC framework, studying properties of message-passing operations can now focus on analyzing the coefficients $\alpha_{(k)}^{(i,j)}$.  
This allows the development of more effective methods within a well-defined framework.  
While we have confirmed the advantages and potential of the LMGC framework, identifying optimal instantiations of LMGCs for specific tasks remains open.

\clearpage

\section*{Acknowledgments}
Part of this research has been funded by the Federal Ministry of Education and Research of Germany and the state of North-Rhine Westphalia as part of the Lamarr-Institute for Machine Learning and Artificial Intelligence and by the Federal Ministry of Education and Research of Germany under grant no. 01IS22094E WEST-AI. Simulations were performed with computing resources granted by WestAI under project rwth1631.
\bibliographystyle{named}
\bibliography{ijcai25}

\clearpage

\begin{figure*}[tb]
     \centering
     \begin{subfigure}[t]{0.32\textwidth}
         \centering
        \def\svgwidth{\textwidth}
     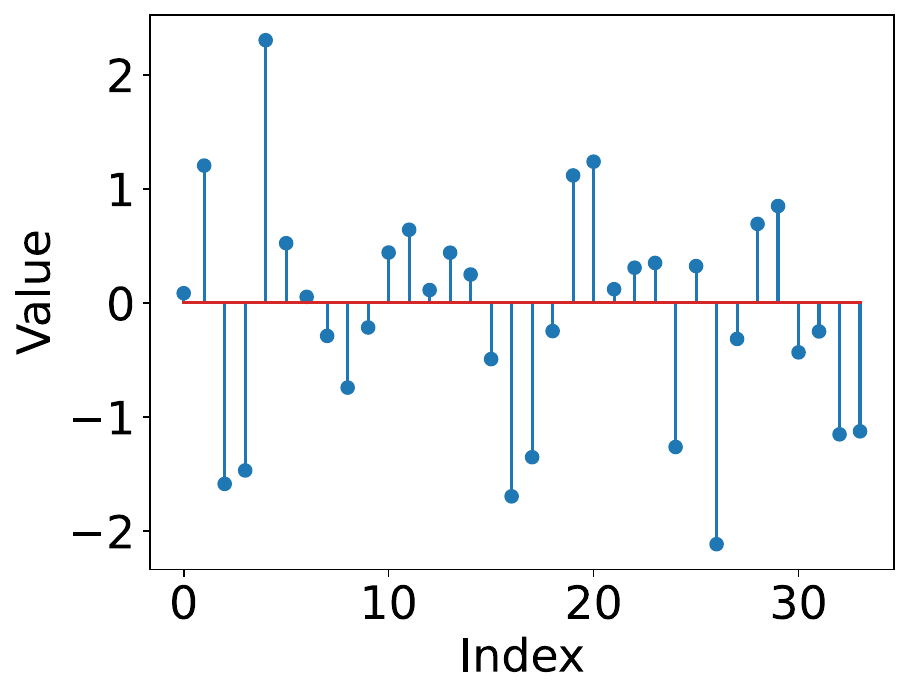
     \caption{Random filter.}
     \end{subfigure}
     \hfill
     \begin{subfigure}[t]{0.32\textwidth}
         \centering
         \def\svgwidth{\textwidth}
         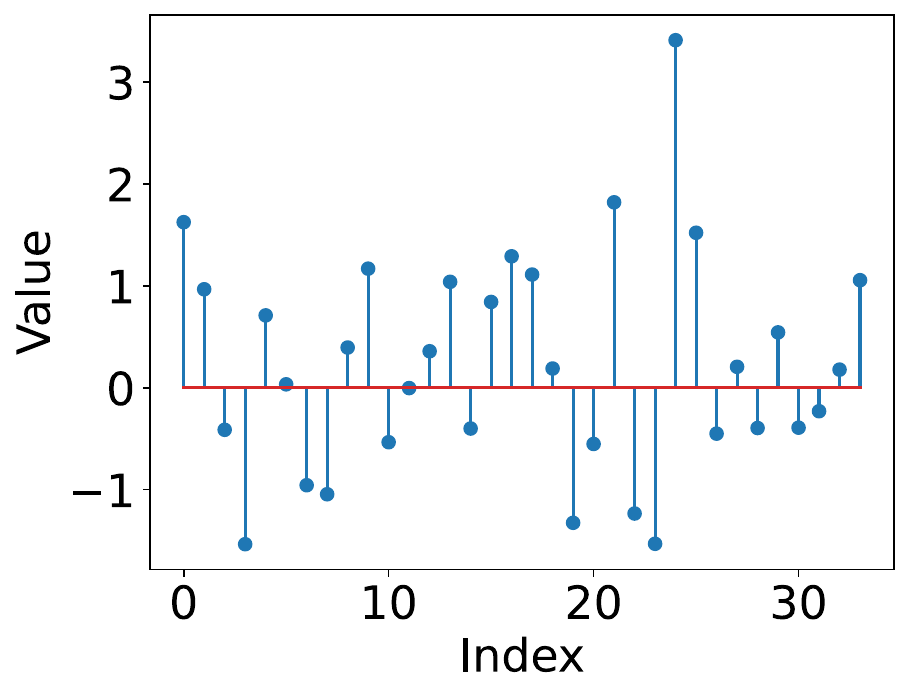
         \caption{Random filter.}
    \end{subfigure}         
    \begin{subfigure}[t]{0.32\textwidth}
         \centering
         \def\svgwidth{\textwidth}
         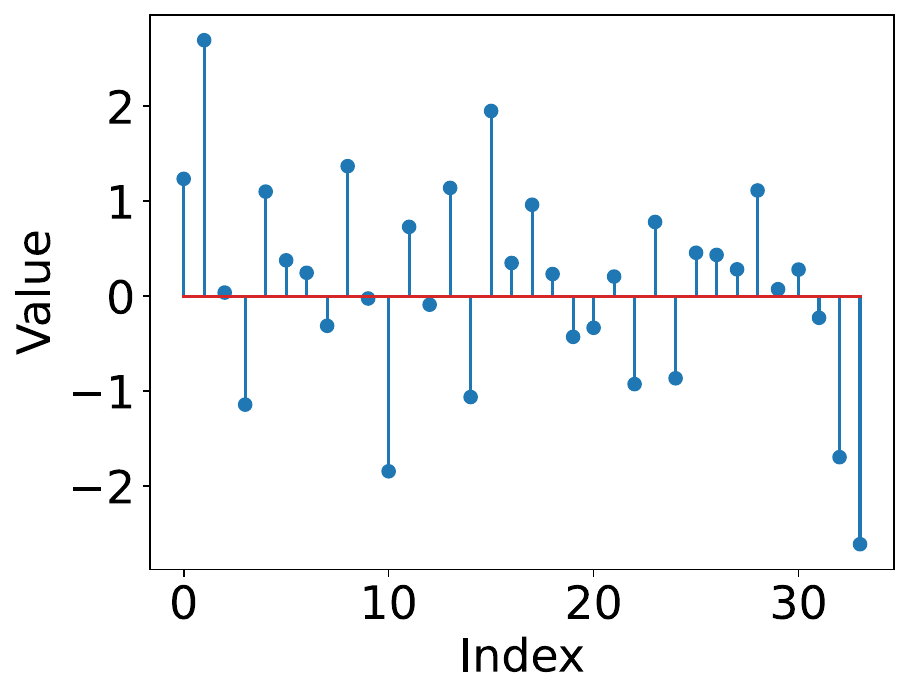
         \caption{Random filter.}
     \end{subfigure}
    \caption{Examples of random filters $F(\vtheta)$ in the graph Fourier domain.}
    \label{fig:conv_filters}
\end{figure*}

\begin{figure*}[tb]
     \centering
     \begin{subfigure}[t]{0.32\textwidth}
         \centering
        \def\svgwidth{\textwidth}
     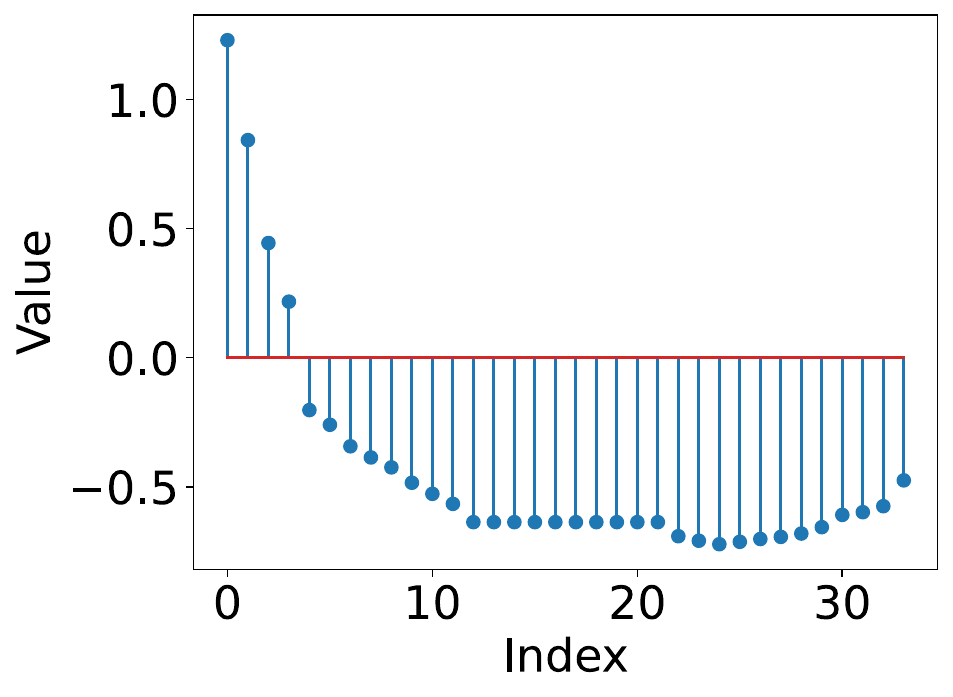
     \caption{$K=2$.}
     \end{subfigure}
     \hfill
     \begin{subfigure}[t]{0.32\textwidth}
         \centering
         \def\svgwidth{\textwidth}
         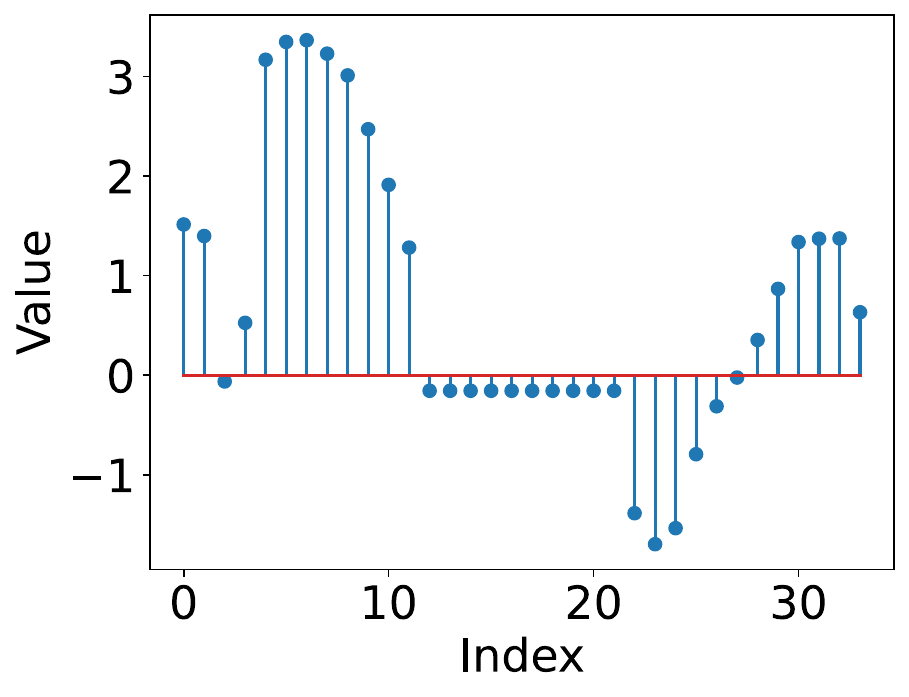
         \caption{$K=8$.}
    \end{subfigure}         
    \begin{subfigure}[t]{0.32\textwidth}
         \centering
         \def\svgwidth{\textwidth}
         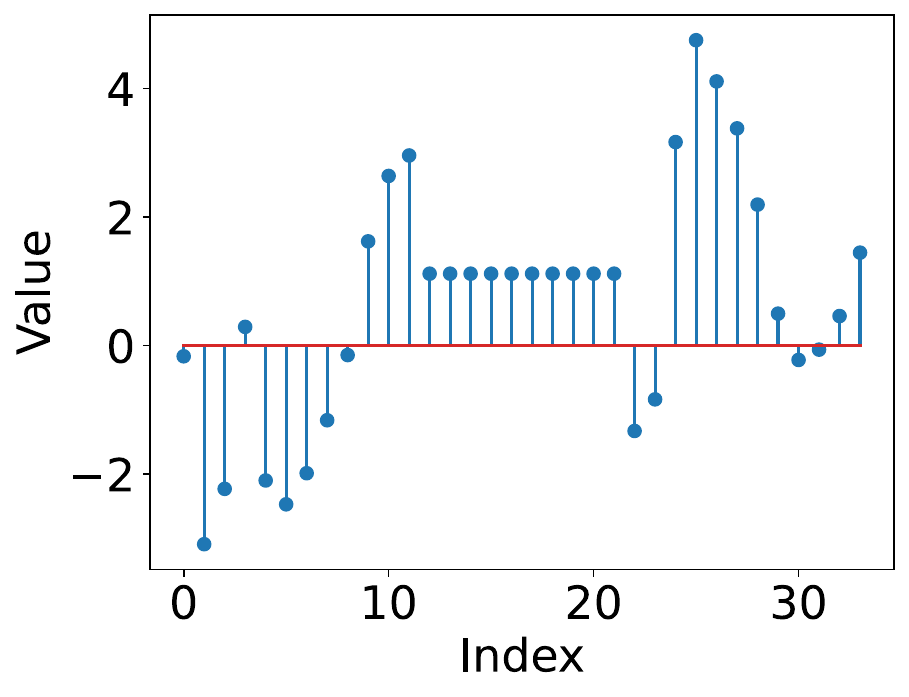
         \caption{$K=16$.}
     \end{subfigure}
    \caption{Spectral filters based on Chebyshev polynomials $F(\vtheta) = \sum_{k=0}^K \evw_k T_k(\vlambda)$ of different degrees $K$.}
    \label{fig:poly_filters}
\end{figure*}

\begin{figure*}[tb]
     \centering
     \begin{subfigure}[t]{0.32\textwidth}
         \centering
        \def\svgwidth{\textwidth}
     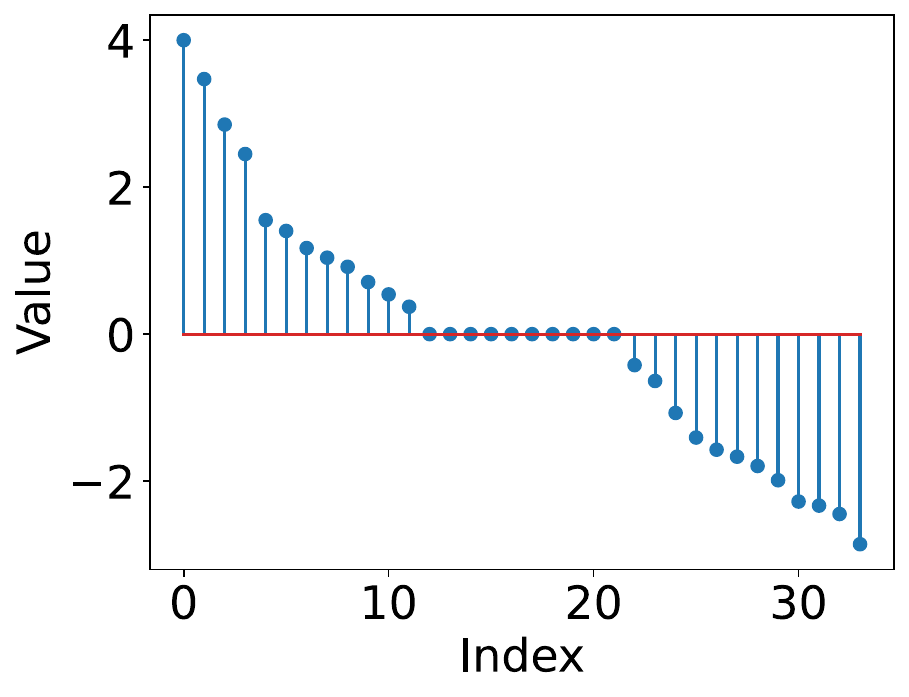
     \caption{$w=4$.}
     \end{subfigure}
     \hfill
     \begin{subfigure}[t]{0.32\textwidth}
         \centering
         \def\svgwidth{\textwidth}
         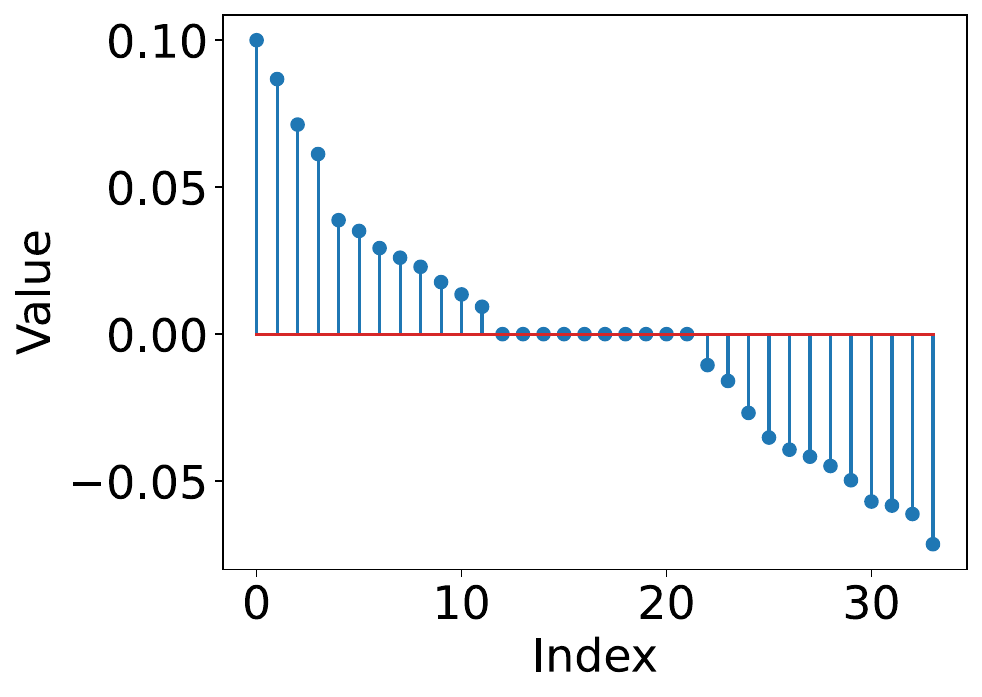
         \caption{$w=0.1$.}
    \end{subfigure}         
    \begin{subfigure}[t]{0.32\textwidth}
         \centering
         \def\svgwidth{\textwidth}
         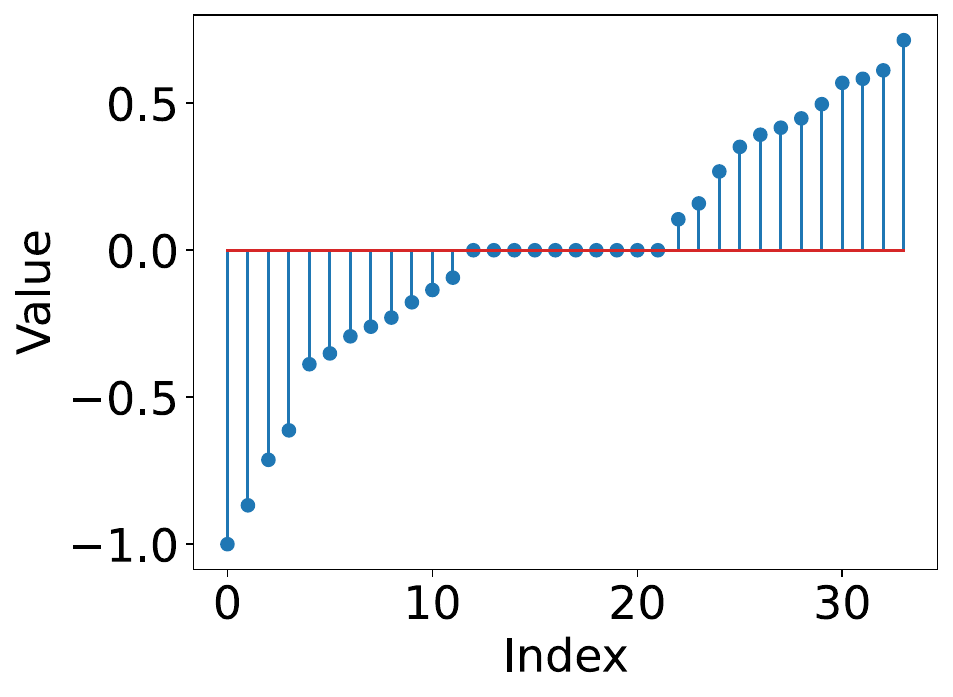
         \caption{$w=-1$.}
     \end{subfigure}
    \caption{Filters of the GCN for $F(\vtheta) = w\vlambda$ with different values for $w$.}
    \label{fig:gcn_filters}
\end{figure*}

\begin{figure*}[tb]
     \centering
     \begin{subfigure}[t]{0.305\textwidth}
         \centering
        \def\svgwidth{\textwidth}
     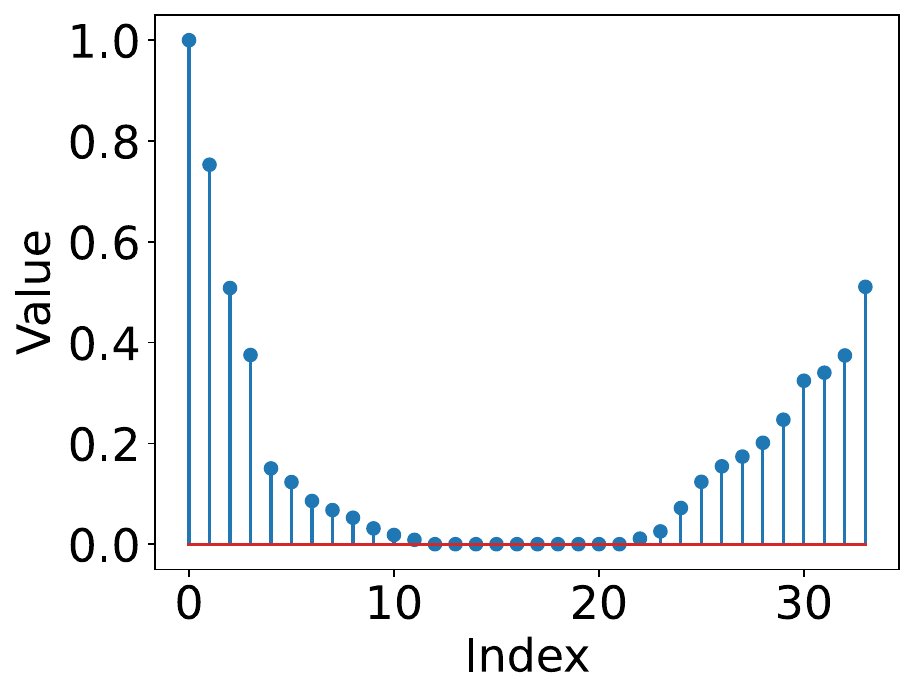
     \caption{$k=2$.}
     \end{subfigure}
     \hfill
     \begin{subfigure}[t]{0.345\textwidth}
         \centering
         \def\svgwidth{\textwidth}
         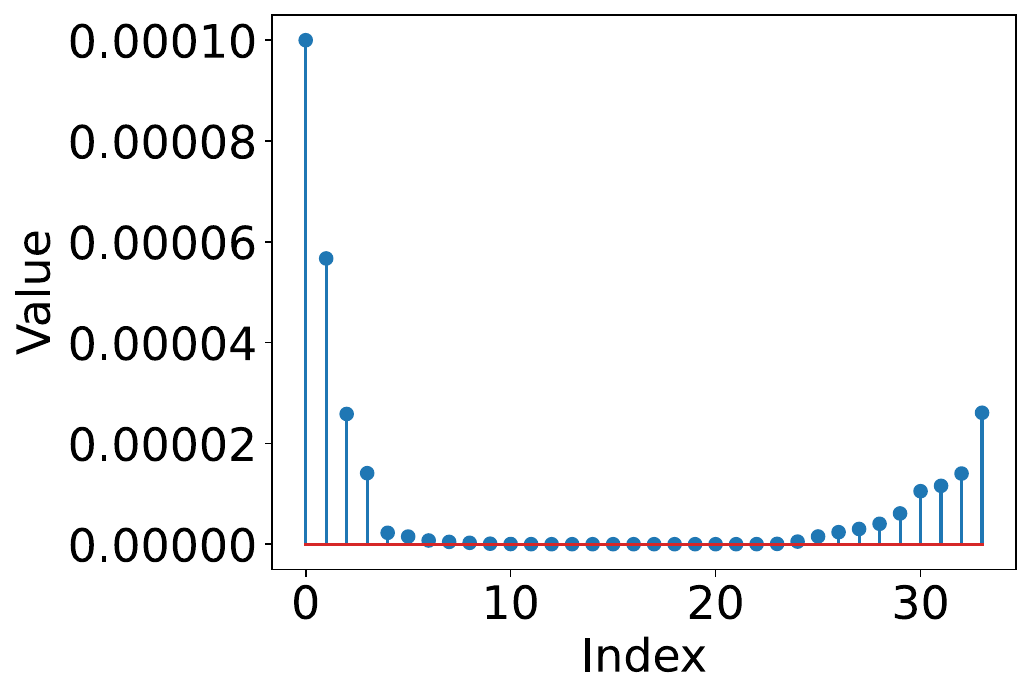
         \caption{$k=4$.}
    \end{subfigure}         
    \begin{subfigure}[t]{0.33\textwidth}
         \centering
         \def\svgwidth{\textwidth}
         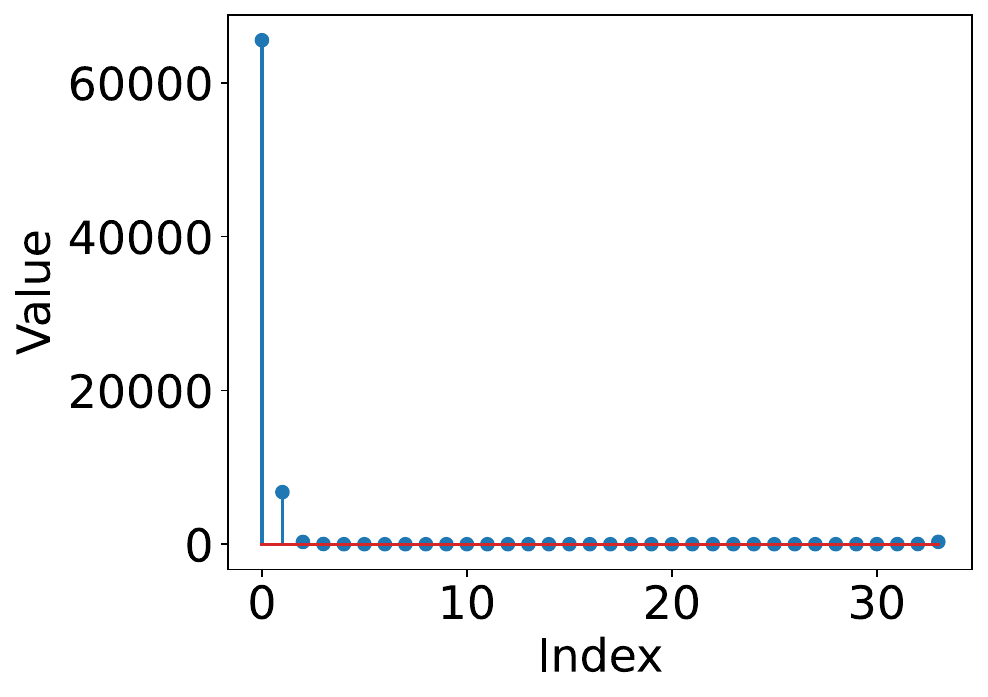
         \caption{$k=16$.}
     \end{subfigure}
    \caption{Combination of $k$ repetitions of random GCN filters $F(\vtheta) = w_k\vlambda \odot \dots \odot w_1\vlambda$ with random values for $w_1,\dots,w_k$.}
    \label{fig:gcn_depth}
\end{figure*}

\begin{table*}[tb]
    \centering
    \begin{tabular}{lcccccccc}
    \toprule
             Task & \multicolumn{2}{c}{Graph Regression} & \multicolumn{6}{c}{Node Classification} \\
        Dataset & ZINC & ZINC12k & Texas & Cornell & Wisconsin & Film & Chameleon & Squirrel\\
        \midrule

         \# of graphs & \num{249456} & \num{12000} & 1 & 1 & 1 & 1 & 1 & 1\\
         avg. \# of nodes & $\sim$ 23.2 & $\sim$ 23.2 & \num{183} & 183 & 251 & \num{7600} & \num{2277} & \num{5201}\\
         avg. \# of edges & $\sim$ 49.8 & $\sim$ 49.8 & 325 & 298 & 515 & \num{30019} & \num{36101} & \num{217073}\\
         avg. node degree & $\sim 2.1$ & $\sim 2.1$ & $\sim 1.8$ & $\sim 1.6$ & $\sim 2.1$ & $\sim 3.9$ & $\sim 15.9$ & $\sim 41.7$\\
         \# of node features & 1 & 1 & \num{1703} & \num{1703} & \num{1703} & 932 & \num{2325} & \num{2089} \\
         \# of classes & 1 & 1 & 5 & 5 & 5 & 5 & 5 & 5\\
         \bottomrule
    \end{tabular}
    \caption{Statistics of all utilized datasets.}
    \label{tab:statistics}
\end{table*}

\begin{table}[tb]
\centering
\begin{tabular}{lr}
\toprule
Method  & Learning rate \\
\midrule
GATv2 & $0.01$ \\
FAGCN & $0.01$ \\
ACM & $0.01$ \\
GIN & $0.01$\\
LMGC & $0.03$ \\
\bottomrule
\end{tabular}
\caption{Optimal learning rate for the results presented in Table 1.}
\label{tab:universality_hyper}
\end{table}

\begin{table*}[tb]
\centering
\begin{tabular}{lrrrrr}
\toprule
Method & Plain & + LapPE & + Jumping Knowledge & + Residual & + All three \\
\midrule
GATv2 & $8/0.001$ & $6/0.001$ & $10/0.001$ & $10/0.001$ & $10/0.001$ \\
FAGCN & $8/0.001$ & $8/0.001$ & $10/0.001$ & $10/0.001$ & $10/0.001$ \\
ACM & $8/0.001$ & $8/0.001$ & $10/0.001$ & $10/0.001$ & $10/0.001$ \\
GIN & $8/0.0003$ & $6/0.001$ & $8/0.0003$ & $10/0.0003$ & $10/0.001$\\
LMGC & $6/0.001$ & $8/0.001$ & $8/0.001$ & $10/0.001$ & $10/0.001$\\
\bottomrule
\end{tabular}
\caption{Optimal hyperparameters for the experiments presented in Table 3. Each entry describes the optimal 'number of layers / learning rate'.}
\label{tab:zinc12k_hyper}
\end{table*}

\begin{table*}[tb]
\centering
\begin{tabular}{lrrrrrr}
\toprule
Method & Texas & Cornell & Wisconsin & Film & Chameleon & Squirrel \\
\midrule
GATv2 & $0.01/0.25$ & $0.01/0.25$ & $0.01/0.25$ & $0.01/0.5$ & $0.01/0.25$ & $0.01/0.0$ \\
FAGCN & $0.003/0.5$ &  $0.01/0.25$ & $0.01/0.25$ & $0.001/0.5$ & $0.001/0.25$ & $0.01/0.25$ \\
ACM & $0.01/0.25$ & $0.01/0.25$ & $0.003/0.25$ & $0.01/0.5$ & $0.01/0.25$ & $0.01/0.25$ \\
GIN & $0.01/0.25$ & $0.01/0.25$ & $0.01/0.25$ & $0.01/0.5$ & $0.003/0.25$ & $0.01/0.25$ \\
LMCGC & $0.01/0.25$ & $0.003/0.25$ & $0.003/0.25$ & $0.001/0.5$ & $0.003/0.0$ & $0.001/0.5$ \\
\bottomrule
\end{tabular}
\caption{Optimal hyperparameters for the experiments presented in Table 4. Each entry describes the optimal 'learning rate / dropout rate'.}
\label{tab:hetero_hyper}
\end{table*}

\begin{table*}[t]
\centering
\begin{tabular}{lrrrrrr}
\toprule
& \multicolumn{3}{c}{ZINC} & \multicolumn{3}{c}{ZINC12k}  \\
 Method  & Train & Test & Time per Epoch (s)& Train & Test & Time per Epoch (s)\\
\midrule
GATv2 & $0.077\pm0.001$ & $0.114\pm0.004$ & $52.7\pm0.0$ & $\underline{0.005}\pm0.003$ & $0.377\pm0.024$ & $2.5\pm0.1$\\
FAGCN & $0.093\pm0.008$ & $0.130\pm0.003$ & $52.9\pm1.1$ & $0.016\pm0.006$ & $0.365\pm0.018$ & $2.5\pm0.0$\\
ACM & $0.109\pm0.003$& $0.128\pm0.001$ & $73.1\pm0.9$ & $0.019\pm0.005$ & $0.278\pm0.006$ & $3.4\pm0.0$\\
GIN & $\underline{0.068}\pm0.001$ & $\underline{0.088}\pm0.002$ & $\mathbf{35.5}\pm1.0$ & $0.018\pm0.009$ & $\underline{0.272}\pm0.009$ & $\mathbf{1.7}\pm0.0$ \\
LMGC & $\mathbf{0.054}\pm0.001$ & $\mathbf{0.080}\pm0.001$ & $\underline{47.5}\pm0.5$ & $\mathbf{0.002}\pm0.001$ & $\mathbf{0.241}\pm0.018$ & $\underline{2.4}\pm0.3$ \\
\bottomrule
\end{tabular}
\caption{MAE results on ZINC and ZINC12k. Optimal hyperparameters for train and test results are independently obtained. Best scores in \textbf{bold}, second-best \underline{underlined}. Each model uses at most \num{100000} parameters.}
\label{tab:zinc}
\end{table*}

\begin{table*}[t]
\centering
\begin{tabular}{lrrrr}
\toprule
& \multicolumn{2}{c}{ZINC} & \multicolumn{2}{c}{ZINC} \\
 Method  & Train & Test & Train & Test \\
\midrule
GATv2 & $10/0.001$ & $10/0.001$ & $8/0.001$ & $8/0.001$ \\
FAGCN & $8/0.001$ & $8/0.001$ & $6/0.001$ & $8/0.001$\\
ACM & $6/0.001$ & $6/0.001$ & $8/0.001$ & $8/0.001$ \\
GIN & $6/0.0003$ & $6/0.0003$ & $6/0.001$ & $8/0.0003$ \\
LMGC & $8/0.001$ & $8/0.001$ & $6/0.001$ & $6/0.001$ \\
\bottomrule
\end{tabular}
\caption{Optimal hyperparameters for the experiments presented in Table~\ref{tab:zinc}. Each entry describes the optimal 'number of layers / learning rate'.}
\label{tab:zinc_hyper}
\end{table*}

\appendix

\section{Further Details on Shared Component Amplification (SCA)}
We provide further details on shared component amplification, which we used throughout the main paper. Based on the graph Fourier transformed graph convolution $F(\vtheta * \vx) = F(\vtheta) \odot F(\vx)$, we visualize the spectral filter $F(\theta)$ for the general graph convolution, polynomial approximations, and the GCN as a first-order approximation.
In Figure~\ref{fig:conv_filters}, we visualize $F(\theta)$ for random filters $\theta$, as is allowed for the general graph convolution (Equation~\ref{eq:siso_gc}). 
Polynomial filters define $F(\theta)$ as a polynomial function of the eigenvalues of the normalized adjacency matrix or the corresponding graph Laplacian (Equation~\ref{eq:siso_poly}). We visualize random such filters based on Chebyshev polynomials of different degrees in Figure~\ref{fig:poly_filters} and visually observe this polynomial structure. 

The GCN as a first-order polynomial (Equation~\ref{eq:siso_mp}) can similarly be described as a spectral filter $F(\theta) = w\vlambda$ where $w\in\R$ is the sole parameter. We visualize this filter for different values $w$ in Figure~\ref{fig:gcn_filters}. As all filters $w\vlambda$ are equivalent up to scaling, we also observe the visual equivalence of such filters. For any parameter $w$, each component gets amplified in the same way. For any other fixed $\vlambda$, any parameter $w$ would still lead to the same amplification of components.
In the MIMO case, such a filter $w_{(i,j)}\vlambda$ is applied for each combination of input channel $i$ and output channel $j$. As for any $w_{(i,j)}$, the components are amplified in the same way based on the given $\vlambda$, we refer to this phenomenon as shared component amplification.

When such a filter is applied repeatedly for random GCN filters, the component corresponding to the maximal absolute value in $\vlambda$ dominates all other components exponentially. Repeated applications of GCN filters $F(\vtheta_K) \odot\dots\odot F(\vtheta_1)$ are visualized in Figure~\ref{fig:gcn_depth}. We refer to this as component dominance. In the MIMO case, when we have shared component amplification, the same component also dominates each combination of input and output channels with increasingly many repetitions.
This combination is equivalent to rank collapse and over-smoothing when a smooth component is amplified.

\section{Mathematical Details}
In this section, we provide the proofs for all statements in the main paper.
\subsection{Proofs for Section 3.}
\subsubsection{Proof of Theorem 1.}

\begin{proof}
We use the vectorized signal $\hat{\vx} = \mathrm{vec}(\mX)\in\mathbb{R}^{n\cdot c}$ by stacking its columns. The graph Fourier transform on matrices and tensors is applied along the node dimension, i.e., independently on each channel. For matrix $\mX$ this results in
\begin{equation*}
F(\mX) = \mU^T\mX\in\R^{n\times d}
\end{equation*}
and for tensor $\tW$ in 
\begin{equation*}
    \hat{\tW} = F(\tW) = \mU^T\times_1\tW\in\R^{n\times c \times d}
\end{equation*}
where $\times_1$ is the 1-mode tensor matrix product~\cite{kolda2009tensor} that performs the desired broadcasted matrix multiplication. With this, we state the multi-channel graph convolution in the Fourier domain as

\begin{equation*}
    \tW * \mX = \mU(\mU^T\times_1\tW \odot \mU^T\mX) = \mU(\hat{\tW} \odot \mU^T\mX)\, .
\end{equation*}
The element-wise product $\hat{\tW}_k(\mU^T\mX)_k$ is a matrix-vector product.
Similarly to the SISO-GC, we simplify this expression using matrix multiplications. Equivalently to the SISO-GC, this can be achieved by diagonalizing $\hat{\tW}$ into a block matrix of diagonal blocks
$$\mD = \begin{bmatrix}
    \hat{\etW}_{1,0,0} & 0 & 0 & & \hat{\etW}_{1,0,d} & 0 & 0 \\
    0 & \ddots & 0 & \dots & 0 & \ddots & 0 \\
    0 & 0 & \hat{\etW}_{n,0,0} & & 0 & 0 & \hat{\etW}_{n,0,d} \\
    & \vdots & & \ddots & & \vdots \\
    \hat{\etW}_{1,c,0} & 0 & 0 & & \hat{\etW}_{1,c,d} & 0 & 0 \\
    0 & \ddots & 0 & \dots & 0 & \ddots & 0 &  \\
    0 & 0 & \hat{\etW}_{n,c,0} & & 0 & 0 & \hat{\etW}_{n,c,d} \\
\end{bmatrix}$$
$\in\R^{nc\times nd}$
where $\hat{\etW}_{k,i,j}\in\R$. This simplifies the equivalent vectorized form into
\begin{equation*}
    \textrm{vec}(\hat{\tW} \odot \mU^T\mX) = \mD \textrm{vec}(\mU^T\mX) = \mD\left(\mI_d \otimes \mU^T\right)\textrm{vec}(\mX)
\end{equation*}
by utilizing the Kronecker product $\otimes$. The matrix $\mD$ can further be decomposed into a sum of Kronecker products $\mD = \sum_{k=1}^n\hat{\tW}_i\otimes\mI^{(k)}_n$, where for each $\mI^{(k)}_n\in\R^{n\times n}$ all entries are zero, apart from position $k,k$ which is one. This lets us state the full vectorized multi-channel graph convolution as
\begin{equation*}
\begin{split}
    \textrm{vec}(\tW*\mX) &= \left(\mI_n\otimes\mU\right)\left(\sum_{k=1}^n \hat{\tW}_k\otimes \mI_n^{(k)}\right)\left(\mI_n\otimes\mU^T\right)\textrm{vec}(\mX) \\
    &= \left(\sum_{k=1}^n\hat{\tW}_k\otimes \mU_{:,k}(\mU_{:,k})^T\right)\textrm{vec}(\mX)
    \end{split}
\end{equation*}
by using the fact that $\mU\mI_n^{(k)}\mU^T = \mU_{:,k}(\mU_{:,k})^T$. We define $\mW^{(k)} = \hat{\tW}_k^T\in\R^{d\times c}$. Inverting the $\textrm{vec}$ operation allows us to avoid the Kronecker product and state the exact multi-channel graph convolution as
\begin{equation*}
    \mW*\mX = \sum_{k=1}^n\mU_{:,k}(\mU_{:,k})^T\mX\mW^{(k)}\, .
\end{equation*}
This concludes the proof.
\end{proof}

\subsubsection{Proof of Proposition 1.}
\begin{proof}
    We decompose $\mX$ and $\mY$ as a sum of $n$ rank-one matrices based on $\mU$. We have 
    \begin{equation*}
    \mX = \sum_{k=1}^n \mU_{:,k}\va^{(k)}\in\R^{n\times d}
    \end{equation*}
    for $\va^{(k)} = (\mU_{:,k})^T\mX\in\mathbb{R}^{1\times d}$ and 
    \begin{equation*}
    \mY = \sum_{k=1}^n \mU_{:,k}\vb^{(k)}\in\R^{n\times c}    
    \end{equation*}
    for $\vb^{(k)} = (\mU_{:,k})^T\mY\in\R^{1\times c}$. Thus,
    \begin{equation}
    \begin{split}
        \tTheta * \mX 
        &= \sum_{k=1}^n \mU_{:,k}(\mU_{:,k})^T\mU_{:,k}\va^{(k)}\mW^{(k)}  \\
        &= \sum_{k=1}^n \mU_{:,k}\va^{(k)}\mW^{(k)} \\
        &= \sum_{k=1}^n \mU_{:,k}\vb^{(k)} = \mY
        \end{split}
    \end{equation}
    for $(\mW^{(k)})_{m,n} = \frac{\evb^{(k)}_n}{c\cdot\eva^{(k)}_m}$. By our assumption we have that $\eva^{(k)}_m \neq 0$ for all $k\in[n]$ and $m\in[d]$.
\end{proof}

\subsubsection{Proof of Proposition 2.}
\begin{proof}
We utilize the eigendecomposition 
\begin{equation*}
\mA^k = \sum_{j=1}^n \lambda_j^k\mU_{:,j}(\mU_{:,j})^T\, .    
\end{equation*}
We then reformulate the polynomial filter as
    \begin{equation}
    \begin{split}
    \sum_{k=0}^K A^k\mX\mV^{(k)} 
    &= \sum_{k=0}^K (\sum_{j=1}^n\lambda_i^k\mU_{:,j}(\mU_{:,j})^T)\mX\mV^{(k)} \\
    &= \sum_{j=1}^n \mU_{:,j}(\mU_{:,j})^T\mX\sum_{k=0}^K\lambda_i^k\mV^{(k)} \\
    &= \sum_{j=1}^n \mU_{:,j}(\mU_{:,j})^T\mX\mW^{(j)} \\
    &= \tTheta_{\textrm{poly}} * \mX
    \end{split}
    \end{equation}
    where $\mW_{(j)} = \sum_{k=0}^K \lambda_j^k\mV^{(k)}$ and the corresponding $\tTheta_{\textrm{poly}}$
\end{proof}

\subsection{Proofs for Section 4}
We now prove the injectivity and linear independence properties of the LMGC stated by Proposition~\ref{prop:injectivity} and Proposition~\ref{prop:lin_indepence}. We first state and prove a helpful lemma:

\begin{lemma}
\label{lemma:injectivity}
Let $\mathcal{X}$ be a countable set, $K\geq1$, and $\alpha_{(k)}^{(i,j)}\in\R$ a scalar for all $k\in[K],\vx_{(i)},\vx_{(j)}\in\mathcal{X}$. Let all $\alpha_{(k)}^{(i,j)}$ be chosen such that for any two finite multisets $\mathcal{X}_1, \mathcal{X}_2\subset \mathcal{X}$ and any $\vx_p,\vx_q\in \mathcal{X}$ with $\mathcal{X}_1\neq\mathcal{X}_2$ or $\vx_p\neq\vx_q$, there exists a $k\in [K]$ such that
$\sum_{\vx_{(j)}\in \mathcal{X}_1} \alpha_{(k)}^{(p,j)}\vx_{(j)} \neq \sum_{\vx_{(j)}\in \mathcal{X}_2} \alpha_{(k)}^{(q,j)}\vx_{(j)}$.

Then, $\sum_{\vx_{(j)}\in\mathcal{X}_p}\mW_{(i,j)}\vx_{(j)}$ is injective on finite multisets $\mathcal{X}_p\subset \mathcal{X}$ and elements $\vx_i\in\mathcal{X}$ for a.e. $\mW^{(1)},\dots,\mW^{(K)}\in\R^{d\times c}$.
\end{lemma}

\begin{proof}
    We want to show that for any two distinct pairs $\vx_1\in\gX,\gX_1\subset\gX$ and $\vx_2\in\gX,\gX_2\subset\gX$, their difference expressions
    \begin{equation}
        \sum_{\vx_p\in \gX_1} \mW_{(1,p)}\vx_p - \sum_{x_q\in \gX_2}\mW_{(2,q)}\vx_q \neq \vzero
    \end{equation}
    is nonzero.
    Substituting the definitions of $\mW_{(1,p)}$ and $\mW_{(2,q)}$, we obtain:
    \begin{equation}
        \sum_{k=1}^K \mW^{(k)}(\sum_{\vx_p\in \gX_1} \alpha_{(k)}^{(1,p)}\vx_p - \sum_{\vx_q\in \gX_2} \alpha_{(k)}^{(2,q)}\vx_q) \neq \vzero\, .
    \end{equation}
    For a.e. $\mW^{(1)},\dots,\mW^{(K)}$, this is zero only when all differences are zero.
    As such, we require all terms to be zero:
    \begin{multline}
        \sum_{\vx_p\in \gX_1} \alpha_{(1)}^{(1,p)}\vx_p - \sum_{\vx_q\in \gX_2} \alpha_{(1)}^{(2,q)}\vx_q \neq 0 \\\land \dots \land \sum_{\vx_p\in \gX_1} \alpha_{(K)}^{(1,p)}\vx_p - \sum_{\vx_q\in \gX_2} \alpha_{(K)}^{(2,q)}\vx_q \neq 0\, .
    \end{multline}
     As $\gX$ is countable, a countable union of measure-zero sets has measure zero. This concludes the proof.
    
\end{proof}

\subsubsection{Proof of Proposition 3.}
\begin{proof}
Continuing with the proof of Lemma~\ref{lemma:injectivity}, we need to show that 
\begin{equation}
\label{eq:prop3_start}
    \sum_{\vx_p\in \gX_1} \alpha_{(k)}^{(1,p)}\vx_p - \sum_{\vx_q\in \gX_2} \alpha_{(k)}^{(2,q)}\vx_q \neq 0
\end{equation}
for some $k\in[K]$ for any $\vx_1,\vx_2\in\gX$ and $\gX_1,\gX_2\subset\gX$ with either $\vx_1\neq\vx_2$ or $\gX_1\neq\gX_2$. We can equivalently state Equation~\ref{eq:prop3_start} in matrix notation $\mX_s\valpha_k \neq 0$ where the matrix $$\mX_s = \begin{bmatrix}
            \vx_{p_1} & \dots & \vx_{p_{|\gX_1|}} & \vx_{q_1} & \dots & \vx_{q_{|\gX_2|}}
        \end{bmatrix}$$
        $\in\mathbb{R}^{d\times (|\gX_1|+|\gX_2|)}$
    is defined as the concatenation of all elements in $\vx_{p_1}, \dots, \vx_{p_{|\gX_1|}}\in\gX_1$ and $\vx_{q_1}, \dots, \vx_{q_{|\gX_2|}}\gX_2$, and the vector $\valpha_{(m)} = \begin{bmatrix}
        \alpha_{(k)}^{(1,p_1)} & \dots & \alpha_{(k)}^{(1,p_{|\gX_1|})} & - \alpha_{(k)}^{(1,q_1)} & \dots & - \alpha_{(k)}^{(1,q_{|\gX_2|})}
    \end{bmatrix}^T$ $\in\mathbb{R}^{|\gX_1|+|\gX_2|}$ for $k\in[K]$.

    $\mX_s\valpha_k = 0$ is true if and only if $\valpha_k\in \textrm{ker}(\mX_s)$. As $\textrm{ker}(\mX_s)$ spans a lower-dimensional subspace, and thus $\mX_s\valpha_k \neq 0$ is satisfied for a.e. choice of $\alpha_{(k)}^{(i,j)}$.
\end{proof}

\subsubsection{Proof of Proposition 4.}
\begin{proof}
Similar to the proofs for Lemma~\ref{lemma:injectivity} and Proposition~\ref{prop:injectivity} we need to show that for any element $\vx_1\vx_2\in\gX$ and multisets $\gX_1\gX_2\subset\gX$ with either $\vx_1 \neq \vx_2$ or $\gX_1 \neq b\cdot\gX_2$, we have
\begin{equation}
    \sum_{\vx_p\in\gX_1} \mW_{(1,p)}\vx_p - c\cdot\sum_{\vx_q\in\gX_2}\mW_{(2,q)}\vx_q\neq 0 
\end{equation}
for any $c\neq 0$ and a.e. $\alpha_{(k)}^{(i,j)}$ and a.e. $\mW^{(k)}$ for all $k\in[K]$. This is equivalent to 
\begin{equation}
    \sum_{k=1}^K \mW^{(k)} \left(\sum_{\vx_p\in\gX_1}\alpha_{(k)}^{(1,p)}\vx_p - c\cdot\sum_{\vx_q\in\gX_2}\alpha_{(k)}^{(2,q)}\vx_q\right) \neq 0
\end{equation}

We proceed with a proof by contradiction.
Assume that the representations are linearly dependent. Then, there exist a constant $c^{(1,2)}\neq 0$, such that
\begin{equation}
\label{eq:sum_lin_depend}
    \sum_{\vx_p\in \gX_1} \alpha_{(k)}^{(1,p)}\vx_p - c^{(1,2)}\cdot\sum_{\vx_q\in \gX_2} \alpha_{(k)}^{(2,q)}\vx_q = 0
\end{equation}
holds for all $k\in[K]$.
This equality can only be satisfied for a.e. choice of coefficients $\alpha_{(k)}^{(i,j)}$ when 
\begin{equation}
\sum_{\vx_p\in \gX_1} \vx_p - c_{k}^{(1,2)}\cdot\sum_{\vx_q\in \gX_2} \vx_q = 0
\end{equation}
for some $c_{k}^{(1,2)}\in\R$, i.e., the sum of the elements is linearly dependent. In this case, the value of $c^{(1,2)}$ is 
\begin{multline}
\label{eq:c_choice}
c^{(1,2)} = \frac{\|\sum_{\vx_p\in \gX_1} \alpha_{(k)}^{(1,p)}\vx_p\|}{\|\sum_{\vx_q\in \gX_2}\alpha_{(k)}^{(2,q)}\vx_q\|} \\ - \frac{(\|\sum_{\vx_p\in \gX_1} \vx_p\| - c_{k}^{(1,2)}\cdot\|\sum_{\vx_q\in \gX_2} \vx_q\|)}{\|\sum_{\vx_q\in \gX_2}\alpha_{(k)}^{(2,q)}\vx_q\|}\, .
\end{multline}
If $\gX_1 = b\cdot\gX_2$ for some scalar $b\neq0$, then $c_k^{(1,2)} = 1$ for all $k\in[K]$, and thus also $c^{(1,2)}=1$. As we assumed $\gX_1 \neq b\cdot\gX_2$, this case cannot occur. 
Under the assumption $\gX_1\neq b\cdot\gX_2$, Equation~\ref{eq:c_choice} can only be satisfied for a single value of $k$ simultaneously for a.e. $\alpha_{(k)}^{(i,j)}$. 
This contradicts the assumption of linear dependence. Thus, for $K>1$, representations are linearly independent.
\end{proof}

\section{Experimental Details}
All experiments were executed on an Nvidia H100 GPU with 96GB memory.

\subsection{Model Architectures}
\paragraph{Graph-Level Prediction}
We evaluate a model that first applies a linear feature encoder that maps the feature dimension onto a hidden dimension $d$. Then, $k$ iterations of message-passing are applied, each followed by a GELU activation. For cases with residual connections, the input to each message-passing layer is added to the output of the GELU activation. This output is stored for all message-passing steps for jumping knowledge. After message-passing, the channel-wise maximum value per node is used when jumping knowledge is allowed. Otherwise, the output of the last layer is used. As a decoder, we use a two-layer MLP with a linear layer followed by a GELU activation and another linear layer.

\paragraph{Node Classification}
The employed model first applies a linear encoder, ReLU activation, and a dropout layer. Then, $k$ layers of message-passing with a residual connection are applied, each followed by a ReLU activation and a dropout layer. As the last operation, a single linear layer is applied. The number of channels for all layers is set to a shared $d$, apart from the initial and final dimensions. For these experiments, $k$ is set to $2$. The number of channels is reduced so that the total number of parameters is maximal but below \num{100000}.

\subsection{Hyperparameters}
For the experiment in Section 6.2, we optimize the learning rate in $\{0.03,0.01,0.003\}$ using a grid search. We repeat each experiment with three fixed seeds and report the average minimal optimization error. Optimal hyperparameters are shown in Table~\ref{tab:universality_hyper}.

For the experiments in Section 6.3, we optimize the learning rate in $\{0.001,0.0003,0.0001\}$ and the number of layers in $\{6,8,10\}$ using a grid search. Each experiment is repeated for four fixed seeds. Average train and test loss are reported. Optimal hyperparameters are shown in Table~\ref{tab:zinc_hyper} and Table~\ref{tab:zinc12k_hyper}.

For the experiments in Section 6.4, we optimize the learning rate in $\{0.01,0.003,0.001\}$ and the dropout ratio in $\{0.0,0.25,0.5\}$ using a grid search. Each experiment is run for ten splits into train, validation and test sets. Based on the best accuracy on the validation set, we reuse the optimal hyperparameters and run all ten splits for five times and report average test scores. Optimal hyperparameters are shown in Table~\ref{tab:hetero_hyper}.

\subsection{Additional Results}
We provide additional results on the full ZINC dataset that contains around $\num{250000}$ molecular graphs, training errors, and runtimes.
Average train and test errors are presented in Table~\ref{tab:zinc}. 
Similarly to the experiment on universality, the LMGC achieves the lowest training error. The difference is more significant on the ZINC12k subset as it is easier to overfit on less data. This improvement also leads to the LMGC having the lowest test loss for ZINC and ZINC12k. Compared to GATv2 and FAGCN, the LMGC improves the test results by at least $40\%$. This confirms the ability of the LMGC to match the expressivity of GIN. 
The training time per epoch of the LMGC is increased compared to GIN by around $35\%$ but slightly reduced compared to GATv2 and FAGCN as the LMGC does not require the normalization step.

\end{document}